\documentclass{article} 
\usepackage{iclr2027_conference,times}

\usepackage{amsmath,amsfonts,bm}

\def\eqref#1{equation~\ref{#1}}

\def\1{\bm{1}}

\DeclareMathAlphabet{\mathsfit}{\encodingdefault}{\sfdefault}{m}{sl}
\SetMathAlphabet{\mathsfit}{bold}{\encodingdefault}{\sfdefault}{bx}{n}

\usepackage[pagebackref]{hyperref}
\usepackage{url}

\usepackage{booktabs}
\usepackage{graphicx}
\usepackage{longtable}
\usepackage{array}
\usepackage{amsmath,amssymb,amsthm}
\usepackage{wrapfig}
\usepackage{placeins}

\theoremstyle{plain}
\newtheorem{theorem}{Theorem}[section]
\newtheorem{proposition}[theorem]{Proposition}
\newtheorem{lemma}[theorem]{Lemma}

\theoremstyle{definition}

\theoremstyle{remark}

\usepackage{enumitem}

\definecolor{iclrblue}{rgb}{0.21,0.49,0.74}
\hypersetup{breaklinks=true,colorlinks=true,citecolor=iclrblue}

\definecolor{electricindigo}{rgb}{0.44, 0.0, 1.0}
\usepackage{multirow}

\definecolor{deblue}{RGB}{11,132,147}
\definecolor{ocra}{RGB}{204, 119, 34}
\usepackage{tikz}

\usepackage[capitalize]{cleveref}
\crefname{section}{Sec.}{Secs.}
\Crefname{section}{Section}{Sections}
\Crefname{table}{Table}{Tables}
\crefname{table}{Tab.}{Tabs.}
\definecolor{indigo(web)}{rgb}{0.29, 0.0, 0.51}

\usepackage{xcolor}

\definecolor{darkblue}{RGB}{40,40,85}
\definecolor{babyblue}{rgb}{0.54, 0.81, 0.94}
\definecolor{pearDark}{HTML}{2980B9}
\definecolor{pearDarker}{HTML}{1D2DEC}

\usepackage{colortbl}
\usepackage{tabulary}
\usepackage{etoolbox}

\usepackage{pifont}

\theoremstyle{plain}

\title{Physics-Informed Neural Plasticity: PDE Solvers That Reshape Themselves}

\author{Chun-Wun Cheng\thanks{These authors contributed equally and work done during a visit in Yau Mathematical Sciences Center} \\
University of Cambridge\\
\texttt{cwc56@cam.ac.uk}
\And
Bingcheng Hu\footnotemark[1] \\
Shanghai Normal University \\
\texttt{1000534759@smail.shnu.edu.cn}
\AND
Angelica I. Aviles-Rivero\thanks{Corresponding author.} \\
YMSC, Tsinghua University\\
\texttt{aviles-rivero@tsinghua.edu.cn}
}

\iclrfinalcopy 
\begin{document}
\addtocontents{toc}{\protect\setcounter{tocdepth}{-1}}

\maketitle
 
\begin{abstract}
Physics-informed neural PDE solvers adapt their parameters to satisfy governing equations, yet their \emph{representational structure typically remains fixed throughout training}. This rigidity is poorly matched to PDE solutions with strongly heterogeneous complexity across space and space--time, leaving capacity insufficient where the physics is difficult and redundant where it is simple. We introduce \emph{physics-informed neural plasticity}, a paradigm in which the representation itself reshapes during optimization in response to unresolved physics. We instantiate this principle with \emph{Representation Capacity Adaptation for PDEs} (ReCAP), a Gaussian-localized solver that dynamically redistributes capacity through local enrichment, residual-directed splitting, gate-based pruning, and function-aware merging. ReCAP uses responsibility-weighted error indicators and the geometry of residual energy to determine where and how to refine. To limit the disturbance introduced by splitting, we introduce \emph{quiet-child refinement}, which initializes new components by transporting the parent representation while controlling instantaneous functional perturbation. We further establish conditional a posteriori reliability and structural-stability guarantees linking localized physics residuals to solution error and stable refinement. Across five challenging 3D and 4D PDE benchmarks against 11 physics-informed solvers, ReCAP achieves the lowest relative $L^2$ error on every problem, reducing error by $10.7\%$--$27.5\%$ relative to the strongest competing result. These results suggest that physics-informed solvers need not merely learn their parameters---they can learn how their representational capacity should be organized.
\end{abstract}

\section{Introduction}
\label{sec:introduction}

Physics-informed neural networks (PINNs) approximate solutions to partial differential equations (PDEs) by optimizing differentiable function representations against governing equations and available initial or boundary data. Their mesh-free formulation and unified treatment of forward and inverse problems make them attractive when conventional discretization is difficult or solution data are scarce \citep{raissi2019physics,karniadakis2021physics}. Yet accurate training remains challenging for high-frequency, multiscale, and strongly nonlinear solutions, where spectral bias and optimization difficulties can leave localized structures 
unresolved even when the overall loss appears small \citep{rahaman2019spectral,wang2021understanding,krishnapriyan2021characterizing}. This challenge is especially pronounced when solution complexity varies across space or space--time: smooth regions may coexist with oscillatory waves, localized peaks, or rapidly evolving fronts that demand markedly different approximation resources.

Yet most existing approaches share a key limitation: \emph{the parameters learn, but the representation itself largely stays fixed}. Optimization strategies, loss balancing, causal training, and residual-adaptive sampling can substantially improve PINN training \citep{wu2023comprehensive,daw2022mitigating,wang2024piratenets}, but they generally operate within a representational budget chosen before training. Adaptive sampling changes \emph{where} the PDE is enforced, but not \emph{what the model can represent locally}. As a result, capacity may remain redundant in already-resolved regions while being insufficient where the physics is hardest to resolve.

We argue that physics-informed PDE solvers should instead exhibit \emph{neural plasticity}: the ability to reshape their representational capacity during training in response to unresolved physics. Rather than optimizing a fixed representation, a plastic solver jointly learns the solution and reorganizes the capacity used to express it. 
 
Localized neural representations provide a natural basis for such plasticity. FBPINNs decompose the domain among overlapping subnetworks \citep{moseley2023finite}; PIG learns the locations and shapes of Gaussian features while fixing their number at initialization \citep{kang2025pig}; and AB-PINNs introduce learnable RBF-defined subdomains with residual-driven addition of local subnetworks \citep{botvinick2026ab}. 

These developments motivate a broader research question: how can physics-informed solvers coordinate \emph{where capacity is needed, what form it should take, and when it has become redundant}, while limiting disruption to the learned solution?

To realize this principle, we introduce \textbf{Representation Capacity Adaptation for PDEs (ReCAP)}, a Gaussian-localized physics-informed solver whose representation can enrich, split, merge, and prune itself during training. Each component carries a learned center and scale, a contribution gate, and a local polynomial approximation. Component responsibilities turn global physics errors into local adaptation signals, allowing ReCAP to \emph{redistribute} capacity rather than simply increase it. Difficult regions receive higher local polynomial order or additional localized components, while weak or functionally redundant components are removed or consolidated. 

These operations implement a common principle: adapt the local approximation space to unresolved physics through \(p\)-like enrichment, \(h\)-like refinement, and coarsening.

For such plasticity to be useful, structural changes must be both \emph{targeted and stable}. ReCAP uses responsibility-weighted physics errors to identify unresolved components and a residual-weighted covariance to orient splitting along the dominant spread of unresolved residual energy. Merging is function-aware, requiring both effective overlap and agreement of local predictions rather than geometric proximity alone. Structural refinement can nevertheless perturb an already learned solution. We therefore introduce a \emph{quiet-child} update: the trained parent is retained, its polynomial is translated into the child's shifted and rescaled coordinates, and the child is introduced through a small gate. This creates new representational capacity while controlling its immediate effect on the learned function.

We complement ReCAP with theoretical analysis connecting local adaptation signals to global solution accuracy and characterizing the stability of structural refinement. Across challenging three- and four-dimensional PDE benchmarks, our experiments further show that dynamically redistributing representational capacity can outperform both fixed-capacity and competing physics-informed solvers. Together, these results support neural plasticity as a principled approach to adapting \emph{where and how} capacity is deployed during physics-informed learning. Our contributions are:

\begin{itemize}[leftmargin=*, nosep]

    \item
    We formulate \emph{physics-informed neural plasticity}, where representational
    capacity becomes an adaptive variable of PDE optimization rather than a fixed
    design choice, and introduce ReCAP to jointly adapt local approximation order,
    component geometry, and component count.

    \item
    We develop a responsibility-driven adaptation mechanism that couples local
    \(p\)-enrichment with residual-directed \(h\)-like splitting. Splitting is
    oriented by the leading eigenvector of a responsibility- and residual-weighted
    covariance, directing new capacity along the dominant spread of unresolved
    physics.

    \item
    We introduce function-aware coarsening and \emph{quiet-child} refinement:
    merging requires both effective overlap and agreement of local predictions,
    while polynomial coordinate translation and gated child insertion create new
    capacity with controlled instantaneous perturbation of the learned solution.

    \item
    We establish conditional a posteriori reliability connecting mass-weighted
    local physics indicators to solution error under PDE stability and
    partition-of-unity responsibilities. For normalized mixtures with a linear
    transport operator, we further derive a gate-dependent bound on the
    instantaneous loss increase caused by additive quiet-child refinement,
    clarifying how capacity can grow with controlled disruption.

    \item
    Across five 3D/4D PDE benchmarks and 11 competing physics-informed solvers,
    ReCAP achieves the lowest relative \(L^2\) error on every benchmark.

\end{itemize}

\section{Related Work}
\label{sec:related_work}

\paragraph{Physics-informed optimization and sampling.}
PINNs enforce differential equations through a training objective \citep{raissi2019physics,karniadakis2021physics}. Gradient imbalance, spectral bias, and difficult optimization landscapes motivate improved training schemes and architectures \citep{wang2021understanding,rahaman2019spectral,krishnapriyan2021characterizing,cho2023separable,wang2023expert,wang2024piratenets}. Residual-adaptive sampling and retain-resample-release sampling direct equation evaluations toward poorly resolved regions \citep{wu2023comprehensive,daw2022mitigating}. These methods alter training signals or their optimization; ReCAP additionally changes the available local basis. Sampling and representation adaptation are complementary, and the experiments examine their cumulative effect.

\paragraph{Local representations and adaptive decomposition.}
Radial basis functions provide trainable localized approximations \citep{broomhead1988multivariable}. Domain-decomposed PINNs, including XPINNs and FBPINNs, organize approximation over local regions \citep{jagtap2022extended,moseley2023finite}. PIG combines trainable Gaussian features with a lightweight neural network \citep{kang2025pig}, while AB-PINNs learn local subdomains and support residual-driven addition of new ones \citep{botvinick2026ab}. ReCAP builds on this local-representation perspective by coordinating component insertion and removal with local polynomial-order changes. Its splitting rule uses residual geometry, its merge criterion checks local functional agreement, and its child initialization accounts explicitly for the coordinate dependence of polynomial coefficients.

\paragraph{Adaptive approximation spaces.}
Classical $hp$ methods allocate degrees of freedom by changing element size and approximation order \citep{babuvska1994p}; hp-VPINNs bring domain decomposition and $hp$ refinement into variational physics-informed learning \citep{kharazmi2021hp}. ReCAP uses analogous operations on learned Gaussian components rather than a fixed mesh. This analogy motivates the algorithm but does not transfer finite-element convergence guarantees to it. 

In particular, the residual scores used for adaptation are not local solution-error bounds: their reliability requires a stability estimate for the PDE and appropriate mass weighting, as made explicit in Section~\ref{sec:recap-theory}.

\section{Methodology}
\label{sec:method}

ReCAP realizes \emph{physics-informed neural plasticity} by adapting the approximation space together with its trainable parameters. At each stage, the structural state $\mathcal S=(M,p_1,\ldots,p_M)$ specifies the number of Gaussian components and their active polynomial orders; the continuous parameters $\Theta$ specify their locations, scales, gates, and local approximations. Training alternates between optimizing $\Theta$ at fixed $\mathcal S$ and updating the representation using localized physics errors. This separates two complementary tasks: fitting the current approximation space and changing where that space has capacity.

The adaptation rules act at different levels. Polynomial enrichment increases local expressivity without adding components; splitting introduces a more localized component; pruning and merging reduce weakly active or redundant components. The base objective $\mathcal L_{\mathrm{base}}$ determines how the PDE and its conditions are enforced, whereas the adaptation scores determine which components are modified. These scores need not be additional loss terms. Figure~\ref{fig:recap-overview} summarizes the loop. Appendix~\ref{app:construction} details the component updates, and Appendix~\ref{sec:further_theory} contains the theoretical statements and proofs.

\subsection{Representation and local error indicators}
\label{sec:representation}

On a spatial or spatiotemporal domain $\Omega\subset\mathbb R^d$, consider $\mathcal N[u](z)=0$ with initial or boundary conditions $\mathcal B[u]=g$. Here $d$ counts all input coordinates, including time when present, and $r_\Theta=\mathcal N[u_\Theta]$ is the physics residual. For FM3D, $z=(t,x,y)$ and $r_\Theta=\partial_tu_\Theta+a\partial_xu_\Theta+b\partial_yu_\Theta$. The other benchmarks use their respective differential operators, given in Appendix~\ref{sec:dataset}.

\begin{figure}[t!]
  \centering
  \includegraphics[width=\linewidth]{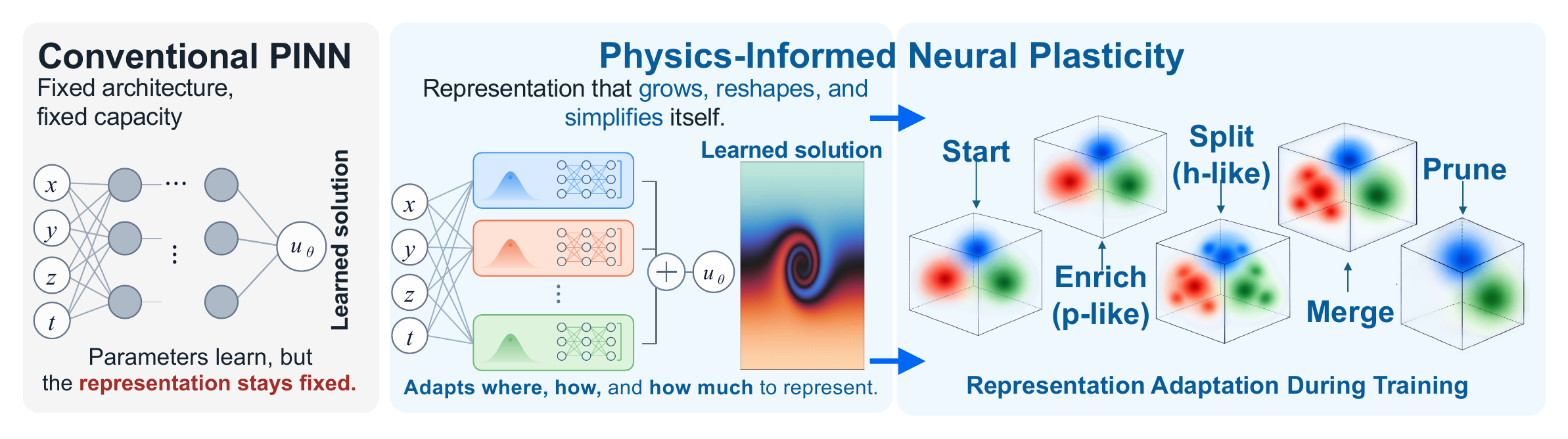}
  \caption{Unlike conventional PINNs with fixed representational capacity, ReCAP reshapes its representation during training through local enrichment, splitting, merging, and pruning, allocating capacity where the PDE requires it.}
  \label{fig:recap-overview}
\end{figure}

Each component $G_m=(\mu_m,\sigma_m,g_m,c_m,p_m)$ has a center, positive diagonal scale, gate logit, polynomial coefficients, and active degree. Fixed input normalization is absorbed into $z$ in the representation formulas; residual evaluation must retain the corresponding chain-rule factors. Define the complete total-degree index set $\mathcal A_p=\{\alpha\in\mathbb N_0^d:|\alpha|\le p\}$ and
\begin{equation}
\xi_m=(z-\mu_m)\oslash\sigma_m,\qquad
\ell_m(z)=\sum_{\alpha\in\mathcal A_{p_m}}c_{m,\alpha}\xi_m^\alpha,
\qquad |\mathcal A_p|=\binom{d+p}{p}.
\label{eq:local-polynomial}
\end{equation}
Here $\xi^\alpha=\prod_q\xi_q^{\alpha_q}$ and $\oslash$ denotes element-wise division. A local degree change therefore changes the number of active basis terms, independently of the number of Gaussians. The described three-dimensional implementation preallocates the complete basis through total degree three, including cross terms, and masks inactive degrees.

Let $\pi_m(z)\ge0$ denote the effective responsibility used to assign evaluation points to component $m$, and write the prediction as $u_\Theta=\mathcal A_\Theta(\{\pi_m,\ell_m,g_m\}_{m=1}^M)$. The update rules use the full prediction and component responsibilities, together with local values or derivatives where required. When internal branches or channels are present, the adaptation statistics average the corresponding quantities across them. The analysis in Section~\ref{sec:recap-theory} states separately when normalized responsibilities or a normalized gated-mixture output are required.

For residual points $\mathcal C=\{z_i\}_{i=1}^{N_c}$ and value observations $\mathcal D_\Gamma=\{(z_i^\Gamma,u_i^\Gamma)\}_{i=1}^{N_\Gamma}$, define
\begin{equation}
\eta_m^{\mathrm{res}}=
\frac{\sum_i\pi_m(z_i)|r_\Theta(z_i)|^2}{\sum_i\pi_m(z_i)+\varepsilon},
\qquad
\eta_m^{\mathrm{bdry}}=
\frac{\sum_i\pi_m(z_i^\Gamma)|u_\Theta(z_i^\Gamma)-u_i^\Gamma|^2}
{\sum_i\pi_m(z_i^\Gamma)+\varepsilon}.
\label{eq:eta-res}
\end{equation}

The normalization separates the \emph{intensity} of local PDE violation from the amount of responsibility mass assigned to a component: a small unresolved region need not be overshadowed by a broad component. The stabilizer $\varepsilon>0$ attenuates scores with negligible support. These are empirical averages over the evaluation points, so their interpretation also depends on how those points are sampled. Initial and boundary observations are combined when initial conditions are imposed softly; exactly enforced conditions contribute no mismatch. The displayed data term applies to prescribed values; derivative conditions require their corresponding operator mismatch. To compare overlapping local approximations, form a $k$-nearest-neighbor graph of centers and evaluate
\begin{equation}
\begin{aligned}
J_{mn}&=\frac{\sum_{i=1}^{N_j}\pi_m(z_i)\pi_n(z_i)
\|\nabla\ell_m(z_i)-\nabla\ell_n(z_i)\|_2^2}
{\sum_{i=1}^{N_j}\pi_m(z_i)\pi_n(z_i)+\varepsilon},\\
\eta_m^{\mathrm{jump}}&=\frac{\sum_{n\in\mathcal N_k(m)}J_{mn}}{|\mathcal N_k(m)|}.
\end{aligned}
\label{eq:pair-jump}
\end{equation}

The gradients are taken with respect to the same global input coordinates, so discrepancies compare the local functions in a common coordinate system. The responsibility product weights the comparison toward overlap, with $\varepsilon$ attenuating pairs of negligible overlap mass. Despite the notation ``jump,'' this is a disagreement score between overlapping models, rather than an interface jump of the global predictor. The adaptation score is $\eta_m=\eta_m^{\mathrm{res}}+\lambda_{\mathrm{bdry}}\eta_m^{\mathrm{bdry}}+\lambda_{\mathrm{jump}}\eta_m^{\mathrm{jump}}$; the weights are nonnegative and either optional term can be disabled. These scores rank candidates; the reliability result below instead uses mass-weighted residual and condition indicators.

\subsection{Polynomial enrichment, pruning, and functional merging}
\label{sec:adaptation}

Let $\widetilde\eta=\operatorname{median}_m\eta_m$. Relative thresholding identifies components whose scores are large compared with the current representation, without prescribing a fixed absolute residual tolerance. A component satisfying $\eta_m>\tau_{\mathrm{split}}\widetilde\eta$ is eligible for $p$-enrichment, $p_m\leftarrow\min(p_m+1,p_{\max})$. When the pre-update degree satisfies $p_m<p_{\max}$, this activates preallocated coefficients while preserving its center, scale, and component count. If the newly activated coefficients are zero, enrichment preserves the predictor exactly at the update and makes additional polynomial directions available to subsequent optimization.

With $\gamma_m=\operatorname{sigmoid}(g_m)$, pruning targets $\gamma_m<\tau_{\mathrm{gate}}$, optionally also requiring $\eta_m<\tau_{\mathrm{merge}}\widetilde\eta$. Candidates are ordered by gate and then error, while retaining at least $M_{\min}$ components. This is a gate-based selection heuristic: in a normalized mixture, a small absolute gate alone does not bound a component's relative influence. Merging instead screens low-error components for both effective overlap and agreement of local predictions. At evaluation points $\{z_q\}_{q=1}^{N_e}$, use
\begin{equation}
O_{ij}=\frac1{N_e}\sum_q\pi_i(z_q)\pi_j(z_q),\qquad
D_{ij}=\frac{\sum_q\pi_i(z_q)\pi_j(z_q)|\ell_i(z_q)-\ell_j(z_q)|^2}
{\sum_q\pi_i(z_q)\pi_j(z_q)+\varepsilon}.
\label{eq:merge-similarity}
\end{equation}
Accept only pairs with $O_{ij}>\tau_{\mathrm{overlap}}$ and $D_{ij}<\tau_{\mathrm{sim}}$, rank them by decreasing overlap, increasing discrepancy, and increasing combined error, then select disjoint pairs. Define the average responsibilities $q_m=N_e^{-1}\sum_q\pi_m(z_q)$ and weights $w_i=q_i/(q_i+q_j+\varepsilon)$ and $w_j=q_j/(q_i+q_j+\varepsilon)$. 

The merged geometry uses the following regularized moment-based update:
\begin{equation}
\begin{aligned}
\mu_{ij}&=w_i\mu_i+w_j\mu_j,\\
\sigma_{ij}^{\odot2}&=w_i[\sigma_i^{\odot2}+(\mu_i-\mu_{ij})^{\odot2}]\\
&\quad+w_j[\sigma_j^{\odot2}+(\mu_j-\mu_{ij})^{\odot2}].
\end{aligned}
\label{eq:merge-moments}
\end{equation}

Here $\odot2$ denotes element-wise squaring. With $\varepsilon=0$ and positive total mass, these formulas match the weighted mean and diagonal variance; for $\varepsilon>0$, the weights sum to less than one. The same weights combine coefficients, $c_{ij}=w_ic_i+w_jc_j$, and gate probabilities, $\gamma_{ij}=w_i\gamma_i+w_j\gamma_j$, before $g_{ij}=\operatorname{logit}(\gamma_{ij})$. Any latent code is averaged likewise, and $p_{ij}=\max(p_i,p_j)$. This merge initialization is heuristic: functional agreement before merging does not ensure an unchanged predictor, because the polynomials use different local coordinates and the responsibilities also change. Appendix~\ref{app:construction} gives the full rules.

\subsection{Residual-directed splitting and coordinate-aware initialization}

\label{sec:splitting}
The remaining capacity $M_{\max}-M$, a fractional budget, and a hard per-event cap limit splitting. The geometric baseline selects the parent's largest-scale axis. The residual-directed alternative uses $\eta_m^{\mathrm{res}}$ for candidate selection and constructs
\begin{equation}
\begin{aligned}
\omega_{im}=\pi_m(z_i)|r_\Theta(z_i)|^2,\qquad
\bar z_m&=\frac{\sum_i\omega_{im}z_i}{\sum_i\omega_{im}+\varepsilon},\\
C_m&=\frac{\sum_i\omega_{im}(z_i-\bar z_m)(z_i-\bar z_m)^\top}
{\sum_i\omega_{im}+\varepsilon}.
\end{aligned}
\label{eq:residual-covariance}
\end{equation}

Choose a unit leading eigenvector $v_m\in\operatorname*{arg\,max}_{\|v\|_2=1}v^\top C_mv$. Responsibilities emphasize the region associated with the parent, while squared residual emphasizes unresolved physics within that region. The resulting direction tracks the dominant spread of residual energy in the normalized input coordinates; it need not align with a coordinate axis or the parent's largest scale. Negligible residual mass or an invalid covariance triggers the geometric fallback. A deterministic subset may be selected using residual energy and responsibility. With exact normalization by $W=\sum_i\omega_{im}>0$, this direction maximizes weighted projected variance (Proposition~\ref{prop:cov_direction}); this characterizes the geometry of the refinement, without asserting that it maximizes subsequent error reduction.

Set $s_m=(v_m^\top\operatorname{diag}(\sigma_m^2)v_m)^{1/2}$ and clamp the tentative displacement $\alpha_{\mathrm{split}}s_mv_m$ in parent-normalized coordinates. The default anchor is the trained parent center; an optional residual-weighted anchor is also clamped. The child's scale is $\operatorname{clip}(\beta_\sigma\sigma_m,\sigma_{\min},\sigma_{\max})$, with $0<\beta_\sigma<1$ and $0<\sigma_{\min}\le\sigma_{\max}$. The full displacement and fallback rules appear in Appendix~\ref{app:construction}.

Because coefficients refer to local coordinates, simply copying them changes the polynomial. Parent and child coordinates instead satisfy
\begin{equation}
\xi_p=d+b\odot\xi_c,\qquad
d=(\mu_c-\mu_p)\oslash\sigma_p,\qquad b=\sigma_c\oslash\sigma_p.
\label{eq:coordinate-transform}
\end{equation}
The multivariate binomial expansion gives
\begin{equation}
c_\beta^c=\sum_{\substack{\alpha\in\mathcal A_{p_p}\\\alpha\ge\beta}}
c_\alpha^p\binom\alpha\beta d^{\alpha-\beta}b^\beta,
\qquad \binom\alpha\beta=\prod_q\binom{\alpha_q}{\beta_q}.
\label{eq:coefficient-translation}
\end{equation}
Here $\alpha\ge\beta$ is component-wise, and the child retains the parent's active degree. The complete total-degree basis is closed under this shift and diagonal rescaling, giving $\ell_c(z)=\ell_p(z)$ (Lemma~\ref{lem:poly_translation}). For example, shifting and rescaling a linear polynomial generally changes both its slope in child coordinates and its constant coefficient; copying the coefficient vector misses this correction. Inactive higher-order coefficients must be zero or excluded, since they can translate into active lower-order terms. Translation therefore preserves the \emph{local function}; controlling the \emph{global predictor} additionally requires controlling the inserted component's weight.

The default additive split retains the parent and initializes a quiet child:
\begin{equation}
\gamma_c=\operatorname{clip}\!\left(
\min\{\rho_g\gamma_p,\gamma_0,\gamma_{\max}\},\gamma_{\min},\gamma_{\max}\right),
\qquad g_c=\operatorname{logit}(\gamma_c),\qquad \rho_g\ll1.
\label{eq:quiet-child}
\end{equation}
The bounds $0<\gamma_{\min}\le\gamma_{\max}<1$ make the logit finite. The gate is intended to limit the child's initial contribution while leaving its geometry and polynomial available for subsequent specialization. Together, translation and gating address different sources of disturbance: translation removes the coordinate-change error, while gating controls the change in aggregation. Section~\ref{sec:recap-theory} quantifies this effect for a normalized mixture. Optional symmetric replacement shifts both offspring oppositely, shrinks and translates both, and divides the gate; it changes the parent as well and is outside that bound.

\subsection{Conditional reliability and controlled adaptation}
\label{sec:recap-theory}

Residual accuracy implies solution accuracy only under stability in the chosen norms \citep{hu2021extended,jiang2026global}, not for arbitrary PDEs or losses \citep{wang2022l2}. For this analysis, let $Q$ be the PDE domain and $\Gamma$ the condition domain, equipped with probability measures $\mu_Q,\mu_\Gamma$; let $\kappa_\Gamma\ge0$. Suppose $Lu=f$, $Bu=g$, and
\begin{equation}
\|v-u\|_X^2\le C_{\mathrm{stab}}^2
\bigl(\|Lv-f\|_{L^2(\mu_Q)}^2+
\kappa_\Gamma\|Bv-g\|_{L^2(\mu_\Gamma)}^2\bigr).
\label{eq:stability_assumption}
\end{equation}
\begin{theorem}[Conditional ReCAP reliability]
\label{thm:main-reliability}
Under \eqref{eq:stability_assumption}, assume $0\le\pi_m\le1$, $\sum_m\pi_m=1$, and $\varepsilon>0$. For masses $s_m=\int_Q\pi_m\,d\mu_Q$ and $s_m^\Gamma=\int_\Gamma\pi_m\,d\mu_\Gamma$, define continuum indicators by the corresponding $\varepsilon$-regularized weighted averages of $|r|^2$ and $|b|^2$, where $r=Lu_\Theta-f$ and $b=Bu_\Theta-g$. Then
\begin{equation}
\|u_\Theta-u\|_X^2\le C_{\mathrm{stab}}^2
\left[\sum_m(s_m+\varepsilon)\eta_m^{\mathrm{res}}
+\kappa_\Gamma\sum_m(s_m^\Gamma+\varepsilon)\eta_m^{\mathrm{bdry}}\right].
\label{eq:recap-reliability}
\end{equation}
The independent-sample certificate, including its finite-sample correction terms and assumptions, is stated and proved in Theorem~\ref{thm:recap-reliability} in Appendix~\ref{sec:further_theory}.
\end{theorem}

The key identity is $\sum_m(s_m+\varepsilon)\eta_m^{\mathrm{res}}=\|r\|_{L^2(\mu_Q)}^2$: mass weighting reconstructs the global residual norm, whereas unweighted scores identify where to adapt. The theorem transfers an assumed PDE stability estimate to these localized quantities; it does not establish that estimate for every benchmark or make the jump score an error bound. The general condition residual $b=Bu_\Theta-g$ reduces to the value mismatch in \eqref{eq:eta-res} when $B$ is a value trace. Independent certification samples are essential: reusing training points requires uniform generalization control \citep{hu2021extended} or continuous-domain verification \citep{eiras2023efficient}. The proof (Theorem~\ref{thm:recap-reliability}) and an explicit transport estimate appear in Appendix~\ref{sec:further_theory}.

\paragraph{Controlling the effect of a new child.}
For one output channel, assume a normalized gated mixture with Gaussian weights $\phi_m$:
\begin{equation}
u^-=\frac{\sum_m q_m\ell_m}{D},\qquad q_m=\gamma_m\phi_m,
\qquad D=\sum_mq_m.
\label{eq:mixture_assumption}
\end{equation}
Additive insertion of a translated child with $\ell_c=\ell_p$ and $q_c=\gamma_c\phi_c$ gives $u^+-u^-=\frac{q_c}{D+q_c}(\ell_p-u^-)$. The change depends on the child's relative weight and parent--mixture disagreement. Other aggregation rules require a separate perturbation analysis.

\begin{theorem}[Additive quiet-child stability]
\label{thm:main-quiet-child}
Let $L=\partial_t+\beta\cdot\nabla$ and use \eqref{eq:mixture_assumption}. Retain the parent unchanged and translate the child so that $\ell_c=\ell_p$. Assume the gate is constant in $z$, $0\le\gamma_c\le\epsilon_g$, and $0\le\phi_c\le1$, $D\ge d_0>0$ on $Q\cup\Gamma$, with $\|L\phi_c\|_{L^\infty(Q)}\le K_\phi$ and $\|LD\|_{L^\infty(Q)}\le K_D$. For $J(v)=\|Lv-f\|_{L^2(Q)}^2+\kappa_\Gamma\|v-g\|_{L^2(\Gamma)}^2$, put $h=\ell_p-u^-$ and $K_\alpha=K_\phi/d_0+K_D/d_0^2$. Define the finite constant
\begin{equation}
C_{\mathrm{split}}=\left[
\left(K_\alpha\|h\|_{L^2(Q)}+d_0^{-1}\|Lh\|_{L^2(Q)}\right)^2
+\kappa_\Gamma d_0^{-2}\|h\|_{L^2(\Gamma)}^2
\right]^{1/2}.
\label{eq:split-constant}
\end{equation}
The post-split predictor satisfies
\begin{equation}
\begin{aligned}
\sqrt{J(u^+)}&\le\sqrt{J(u^-)}+\epsilon_gC_{\mathrm{split}},\\
J(u^+)-J(u^-)&\le2\epsilon_gC_{\mathrm{split}}\sqrt{J(u^-)}
+\epsilon_g^2C_{\mathrm{split}}^2.
\end{aligned}
\label{eq:sqrt_loss_shock}
\end{equation}
If admissible child gates can approach zero, the refined class contains the old class in its closure, so its optimal achievable loss cannot increase.
\end{theorem}

The proof appears with Theorem~\ref{thm:quiet_child} in Appendix~\ref{sec:further_theory}. For controlled $C_{\mathrm{split}}$, a smaller gate limits the immediate loss increase. The constant also exposes the relevant conditions: adequate mixture coverage through $d_0$, controlled child derivatives through $K_\phi$, and agreement between the parent and current mixture through $h$. A small absolute gate alone is therefore insufficient when the denominator is very small or the child is excessively narrow. A fixed positive gate floor also precludes the zero-gate argument for class inclusion. This result concerns additive insertion for the stated transport objective, rather than monotonic training loss or a guarantee for nonlinear and higher-order benchmark operators. Zero-initialized $p$-enrichment preserves the current function exactly; symmetric replacement, pruning, and merging require separate perturbation analysis.

\section{Experimental Results}
\label{sec:results}

\paragraph{Benchmarks, baselines and implementation details.}

We evaluate ReCAP on five PDE benchmarks: flow mixing (FM3D), Helmholtz (H3D), nonlinear diffusion (ND3D), and Klein--Gordon equations (KG3D and KG4D). The dimension counts all input coordinates, including time: H3D has three spatial coordinates, FM3D, ND3D, and KG3D have two spatial coordinates and time, and KG4D has three spatial coordinates and time. Appendix~\ref{sec:dataset} provides the problem definitions. All experiments run on a single NVIDIA RTX 4090, with baselines reimplemented in the same hardware and software environment. We compare against eleven methods: PINN~\citep{raissi2019physics}, ActNet~\citep{guilhoto2026ActNet}, PirateNets~\citep{wang2024piratenets}, SPINN~\citep{cho2023separable}, JAXPI~\citep{wang2023expert}, CoupledNet~\citep{meng2026deep}, CausalPINN~\citep{wang2024respecting}, DB-PINN~\citep{zhoudbpinn}, SINN~\citep{yu2026spectralinformedneuralnetworksoutperform}, Scale-PINN~\citep{chiu2026scalepinnlearningefficientphysicsinformed}, and PIG~\citep{kang2025pig}. ReCAP hyperparameters and computational costs appear in Appendices~\ref{sec:experiment-settings} and~\ref{sec:computational-cost}, respectively. We report relative $L^2$ error to compare solution accuracy across these benchmarks.

\begin{table}[t!] 
\centering 

\caption{Reported relative $L^2$ errors on five PDE benchmarks (lower is better). The lowest error in each column is highlighted in \colorbox[HTML]{E3FFE2}{green}, and the second-lowest is underlined. The final row gives ReCAP's percentage error reduction relative to the strongest baseline, computed from the displayed values. A dash indicates an unreported result.} 
\label{tab:main_table} 
\resizebox{\linewidth}{!}{%
\begin{tabular}{lccccc} 
\toprule 
Method & FM3D & H3D & ND3D & KG3D & KG4D\\ 
\midrule 

PINN~\citep{raissi2019physics}
& 0.001968 & 0.527911 & 0.007788 & 0.042415 & 0.009683 \\ 

SPINN~\citep{cho2023separable} 
& 0.007995 & \underline{0.066970} & 0.005182 & \underline{0.003773} & 0.004786 \\ 

JAXPI~\citep{wang2023expert}
& 0.000328 & 0.110786 & 0.002229 & 0.007136 & \underline{0.003942} \\ 

PirateNets~\citep{wang2024piratenets} 
& \underline{0.000291} & 0.077632 & 0.002505 & 0.003860 & 0.007206 \\ 

CausalPINN~\citep{wang2024respecting} 
& 0.000383 & - & 0.006638 & 0.021933 & 0.008271 \\ 

DB-PINN~\citep{zhoudbpinn} 
& 0.002053 & 0.587973 & 0.005622 & 0.007994 & 0.017544 \\ 

PIG~\citep{kang2025pig} 
& 0.000305 & 0.184739 & \underline{0.001791} & 0.004206 & 0.005286 \\ 

ActNet~\citep{guilhoto2026ActNet}
& 0.001787 & 0.093405 & 0.004842 & 0.006199 & 0.005991 \\ 

CoupledNet~\citep{meng2026deep} 
& 0.000856 & 0.158846 & 0.006488 & 0.070929 & 0.009005 \\ 

SINN~\citep{yu2026spectralinformedneuralnetworksoutperform} 
& 0.003955 & 0.700347 & 0.003598 & 0.011043 & 0.059185 \\ 

Scale-PINN~\citep{chiu2026scalepinnlearningefficientphysicsinformed} 
& 0.034121 & 0.565737 & 0.082126 & 0.092426 & 0.093574 \\ 

\midrule
ReCAP (Ours) 
& \cellcolor[HTML]{E3FFE2}0.000211 & \cellcolor[HTML]{E3FFE2}0.059810 & \cellcolor[HTML]{E3FFE2}0.001574 & \cellcolor[HTML]{E3FFE2}0.003181 & \cellcolor[HTML]{E3FFE2}0.003463\\ 

Improvement 
& 27.5\% & 10.7\% & 12.1\% & 15.7\% & 12.2\%\\

\bottomrule 
\end{tabular} 
}
\end{table}
\begin{figure}[t!]
    \centering
    \includegraphics[width=\linewidth]{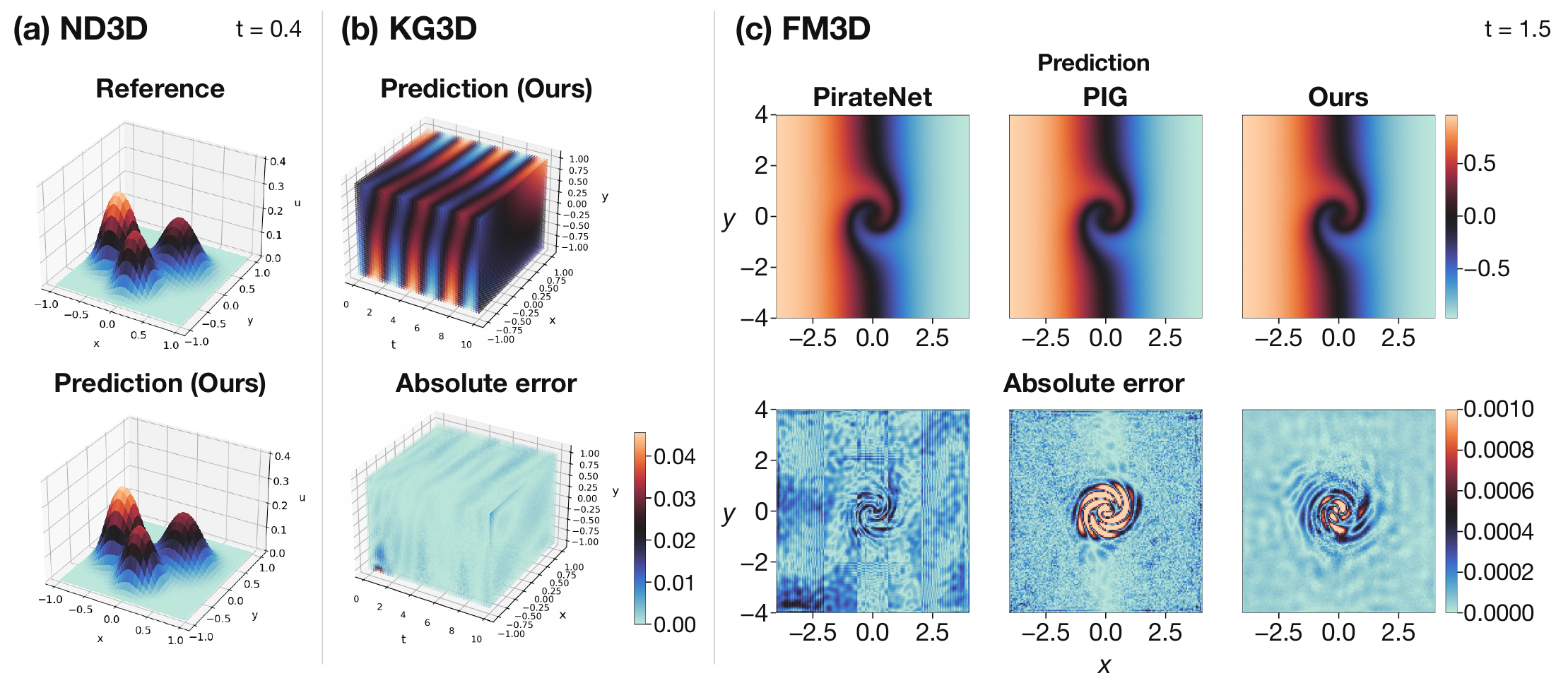}
    \caption{\textbf{Qualitative solution and error comparisons.}
    (a) ND3D at $t=0.4$: reference solution and ReCAP prediction.
    (b) KG3D: ReCAP prediction and pointwise absolute error.
    (c) FM3D at $t=1.5$: predictions and corresponding absolute-error maps for PirateNets, PIG, and ReCAP.
    ReCAP accurately resolves the dominant solution structures while maintaining small and spatially localized errors, including around rapidly varying regions.}
    \label{fig:error}
\end{figure}

\textbf{Main Results.}
ReCAP achieves the lowest reported relative $L^2$ error on each of the five benchmarks in Table~\ref{tab:main_table}. Compared with the strongest baseline on each problem, ReCAP reduces error by $27.5\%$ on FM3D (PirateNets), $10.7\%$ on H3D (SPINN), $12.1\%$ on ND3D (PIG), $15.7\%$ on KG3D (SPINN), and $12.2\%$ on KG4D (JAXPI). These comparisons establish the improvement on the tested three- and four-dimensional problems; the ablations below examine how the representation and its adaptation contribute to the observed behavior.

The strongest baseline varies across the benchmarks: PirateNets, SPINN, PIG, and JAXPI each provide the best competing result on at least one problem. ReCAP improves over these problem-specific reference points across transport, elliptic, nonlinear diffusion, and wave equations. This breadth supports its effectiveness on the tested problems; the ablations below examine whether the gains depend on adaptive allocation beyond component-count scaling.

Figure~\ref{fig:error} provides spatial evidence complementary to the aggregate error metrics. 
On ND3D, ReCAP closely reproduces the reference solution, including the localized peaks and surrounding low-amplitude structure. 
For KG3D, the predicted field preserves the pronounced spatiotemporal variations while the corresponding pointwise error remains small over most of the domain. 
The FM3D comparison further highlights the difference between methods: although all three models recover the overall spiral structure, ReCAP exhibits substantially less background error and concentrates the remaining error near the most rapidly varying region around the spiral core and fronts. 

These visualizations complement the aggregate errors by showing how ReCAP resolves localized solution structures and where prediction errors remain.
Additional solution visualizations are provided in Appendix~\ref{sec:vis}.

\begin{figure}[t!]
    \centering
    \includegraphics[width=\linewidth]{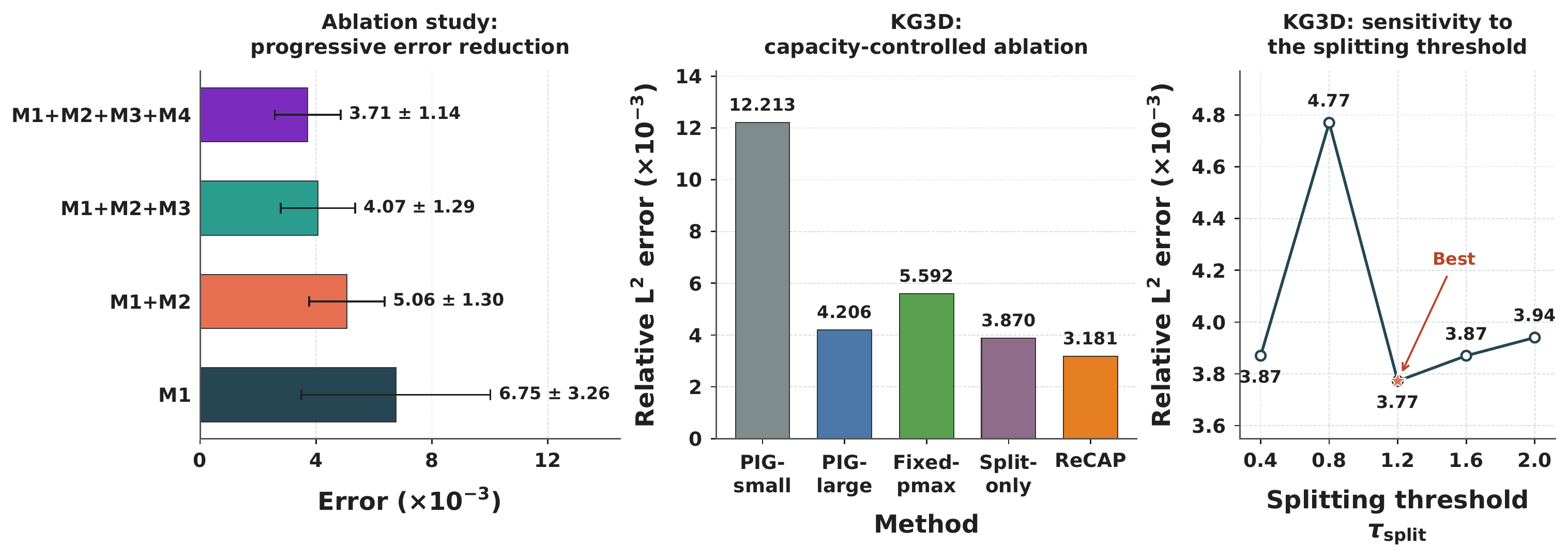}
    \caption{KG3D ablations. Left: cumulative error reduction from adding M2--M4 to the base representation M1. Middle: capacity-controlled comparisons. Right: splitting-threshold sensitivity, with $\tau_{\mathrm{split}}=1.2$ giving the lowest error among the tested values.}
    \label{fig:ab_vis}
\end{figure}

\paragraph{Ablation study.} Figure~\ref{fig:ab_vis} examines ReCAP's key design choices on KG3D, including (i) the cumulative contributions of its adaptation components, (ii) the benefits of adaptive capacity allocation beyond particle-count scaling, and (iii) sensitivity to the splitting threshold.

\paragraph{Cumulative component ablation.}
We first evaluate the incremental benefits of the adaptation mechanisms on KG3D (Figure~\ref{fig:ab_vis}, left). Starting from M1, a dynamic Gaussian-particle representation with hard constraints, we successively introduce particle splitting for local refinement (M2), pruning and merging for redundancy control (M3), and residual-based adaptive collocation sampling (M4). The relative $L^2$ error decreases at every stage, from $6.75\times10^{-3}$ to $5.06\times10^{-3}$, $4.07\times10^{-3}$, and finally $3.71\times10^{-3}$, corresponding to a $45.0\%$ reduction over M1. 

The reductions continue after splitting, supporting refinement together with redundancy control and adaptive sampling. Since the components are introduced cumulatively, the observed reductions quantify incremental gains conditional on preceding components, rather than their independent effects.

\paragraph{Adaptive allocation versus particle-count scaling.}
We next examine whether increasing the particle count is sufficient to recover ReCAP's accuracy (Figure~\ref{fig:ab_vis}, middle). PIG-small uses the initial Gaussian count, whereas PIG-large matches the final count reached by ReCAP. Two additional controls restrict how capacity is allocated: Fixed-$p_{\max}$ assigns the maximum polynomial order uniformly, while Split-only uses adaptive splitting without polynomial-order refinement or coarsening. ReCAP achieves a relative $L^2$ error of $3.181\times10^{-3}$, reducing the error by $24.4\%$ relative to PIG-large, $43.1\%$ relative to Fixed-$p_{\max}$, and $17.8\%$ relative to Split-only. Thus, matching the final particle count does not close the accuracy gap, and neither uniformly maximal polynomial order nor splitting alone matches the complete method. Together, these comparisons support ReCAP's combined strategy of local spatial refinement, polynomial-order adaptation, and redundancy control over the tested alternatives.

\paragraph{Splitting-threshold sensitivity.}
We evaluate sensitivity to the splitting threshold over
$\tau_{\mathrm{split}}\in\{0.4,0.8,1.2,1.6,2.0\}$
on KG3D (Figure~\ref{fig:ab_vis}, right). The best tested setting, $\tau_{\mathrm{split}}=1.2$, achieves a relative $L^2$ error of $3.772\times10^{-3}$, a $20.9\%$ reduction relative to $\tau_{\mathrm{split}}=0.8$. The response is non-monotonic: although $0.8$ produces the largest error, the remaining thresholds, $0.4$, $1.6$, and $2.0$, yield errors between $3.870\times10^{-3}$ and $3.940\times10^{-3}$, all within $4.5\%$ of the minimum. 

The sweep therefore favors $\tau_{\mathrm{split}}=1.2$, with several other thresholds giving similar accuracy. These values correspond to the sensitivity sweep; the main-benchmark results are reported separately in Table~\ref{tab:main_table}.

\begin{samepage}
\section{Conclusion}
ReCAP treats representational capacity as an adaptive resource for physics-informed PDE solving. It couples residual-guided polynomial enrichment and splitting with function-aware coarsening, using coordinate-aware quiet-child initialization to introduce new capacity. Across five three- and four-dimensional PDE benchmarks, ReCAP achieves the lowest reported relative $L^2$ error; capacity-controlled comparisons further support the benefit of targeted allocation beyond increasing component count alone. Under the stated assumptions, the analysis connects mass-weighted physics indicators to solution error and bounds the perturbation from additive refinement. Together, these results support physics-informed neural plasticity: jointly learning the solution and adapting where, how, and how much capacity is used to represent it.
\end{samepage}


\subsection*{AI use statement}
In this work, we used generative AI tools to assist in developing and writing proofs of mathematical claims, provide feedback on methodological exposition, and refine the interpretation and presentation of results. Additionally, we used these tools to polish the writing, improve clarity and organization, We have reviewed all AI-assisted work, including checking the assumptions, mathematical statements, and individual steps of the proofs, and verifying that the methodological descriptions and interpretations are consistent with our implementation and experimental evidence. We take responsibility for the final content of this work, including all text, claims, and artifacts produced with the aid of generative AI.

\subsection*{Acknowledgments}
AIAR gratefully acknowledges support from the Tsinghua University Dushi Program and the Tsinghua University Initiative Scientific Research Program. This work is also supported by the Yau Mathematical Sciences Center, Tsinghua University.
CWC was supported by the Swiss National Science Foundation (SNSF) under Grant No. 20HW-1 220785.

\bibliography{iclr2026_conference}

@article{tamamidis1993evaluation,
  title   = {Evaluation of Various High-Order-Accuracy Schemes with and without Flux Limiters},
  author  = {Tamamidis, Panos and Assanis, Dennis N.},
  journal = {International Journal for Numerical Methods in Fluids},
  volume  = {16},
  number  = {10},
  pages   = {931--948},
  year    = {1993},
}

@article{raissi2019physics,
  title={Physics-informed neural networks: A deep learning framework for solving forward and inverse problems involving nonlinear partial differential equations},
  author={Raissi, Maziar and Perdikaris, Paris and Karniadakis, George E.},
  journal={Journal of Computational Physics},
  volume={378},
  pages={686--707},
  year={2019},
  publisher={Elsevier},
  doi={10.1016/j.jcp.2018.10.045},
}

@article{karniadakis2021physics,
  title={Physics-informed machine learning},
  author={Karniadakis, George E and Kevrekidis, Ioannis G and Lu, Lu and Perdikaris, Paris and Wang, Sifan and Yang, Liu},
  journal={Nature Reviews Physics},
  volume={3},
  number={6},
  pages={422--440},
  year={2021},
  publisher={Nature Publishing Group},
  doi={10.1038/s42254-021-00314-5},
}

@article{jagtap2022extended,
  title={Extended physics-informed neural networks ({XPINNs}): A generalized space-time domain decomposition based deep learning framework for nonlinear partial differential equations},
  author={Jagtap, Ameya D and Karniadakis, George E},
  journal={Communications in Computational Physics},
  volume={28},
  pages={2002--2041},
  year={2020},
  number={5},
  doi={10.4208/cicp.OA-2020-0164},
}

@article{wang2021understanding,
author = {Wang, Sifan and Teng, Yujun and Perdikaris, Paris},
title = {Understanding and Mitigating Gradient Flow Pathologies in Physics-Informed Neural Networks},
journal = {SIAM Journal on Scientific Computing},
volume = {43},
number = {5},
  pages={A3055--A3081},
year = {2021},
  doi={10.1137/20M1318043},
}

@article{krishnapriyan2021characterizing,
  title={Characterizing possible failure modes in physics-informed neural networks},
  author={Krishnapriyan, Aditi S. and Gholami, Amir and Zhe, Shandian and Kirby, Robert and Mahoney, Michael W},
  journal={Advances in Neural Information Processing Systems},
  volume={34},
  year={2021},
}

@inproceedings{kang2025pig,
  title={{PIG}: {P}hysics-{I}nformed {G}aussians as {A}daptive {P}arametric {M}esh {R}epresentations},
  author={Namgyu Kang and Jaemin Oh and Youngjoon Hong and Eunbyung Park},
  booktitle={The Thirteenth International Conference on Learning Representations},
  year={2025},
  url={https://openreview.net/forum?id=y5B0ca4mjt},
}

@article{broomhead1988multivariable,
  title={Multivariable functional interpolation and adaptive networks},
  author={Broomhead, David S. and Lowe, David},
  journal={Complex Systems},
  volume={2},
  pages={321--355},
  year={1988},
}

@inproceedings{guilhoto2026ActNet,
  title={Deep Learning Alternatives of the {Kolmogorov} Superposition Theorem},
      author={Leonardo Ferreira Guilhoto and Paris Perdikaris},
  year={2025},
  url={https://proceedings.iclr.cc/paper_files/paper/2025/hash/c2469e35d469e3c0eca09dbe484eb474-Abstract-Conference.html},
  booktitle={The Thirteenth International Conference on Learning Representations},
}

@article{wang2023expert,
  title={An Expert's Guide to Training Physics-informed Neural Networks},
  author={Wang, Sifan and Sankaran, Shyam and Wang, Hanwen and Perdikaris, Paris},
  journal={arXiv preprint arXiv:2308.08468},
  year={2023},
}

@article{cho2023separable,
  title={Separable Physics-Informed Neural Networks},
  author={Cho, Junwoo and Nam, Seungtae and Yang, Hyunmo and Yun, Seok-Bae and Hong, Youngjoon and Park, Eunbyung},
  journal={Advances in Neural Information Processing Systems},
  year={2023},
  volume={36},
  pages={23761--23788},
}

@inproceedings{meng2026deep,
  title={Deep Coupling Learning for Solving {PDEs}},
  author={Meng, Lingshi and Shi, Haosen and Pan, Sinno Jialin},
  booktitle={Proceedings of the 43rd International Conference on Machine Learning},
  year={2026},
  url={https://openreview.net/forum?id=4pfxoYPQSY},
}

@article{wang2024respecting,
  title={Respecting causality for training physics-informed neural networks},
  author={Wang, Sifan and Sankaran, Shyam and Perdikaris, Paris},
  journal={Computer Methods in Applied Mechanics and Engineering},
  volume={421},
  pages={116813},
  year={2024},
  publisher={Elsevier},
  doi={10.1016/j.cma.2024.116813},
}

@inproceedings{zhoudbpinn,
author = {Zhou, Chenhong and Chen, Jie and Yang, Zaifeng and Png, Ching Eng},
title = {Dual-balancing for physics-informed neural networks},
year = {2025},
booktitle = {Proceedings of the Thirty-Fourth International Joint Conference on Artificial Intelligence},
series = {IJCAI '25},
  pages={7164--7172},
  doi={10.24963/ijcai.2025/797},
  url={https://www.ijcai.org/proceedings/2025/797},
}

@inproceedings{
yu2026spectralinformedneuralnetworksoutperform,
title={Spectral-Informed Neural Networks Outperform Spectral methods in High-dimensional {PDE}s},
author={Tianchi Yu and Ivan Oseledets},
booktitle={Forty-third International Conference on Machine Learning},
year={2026},
url={https://openreview.net/forum?id=KAHCMPsPeI}
}

@misc{chiu2026scalepinnlearningefficientphysicsinformed,
  title={{Scale-PINN}: Learning Efficient Physics-Informed Neural Networks Through Sequential Correction},
      author={Pao-Hsiung Chiu and Jian Cheng Wong and Chin Chun Ooi and Chang Wei and Yuchen Fan and Yew-Soon Ong},
      year={2026},
      eprint={2602.19475},
      archivePrefix={arXiv},
      primaryClass={cs.CE},
      url={https://arxiv.org/abs/2602.19475}, 
}

@article{wang2022l2,
  title={Is {$L^2$} Physics Informed Loss Always Suitable for Training Physics Informed Neural Network?},
  author={Wang, Chuwei and Li, Shanda and He, Di and Wang, Liwei},
  journal={Advances in Neural Information Processing Systems},
  volume={35},
  pages={8278--8290},
  year={2022},
}

@article{hu2021extended,
  title={When do extended physics-informed neural networks ({XPINNs}) improve generalization?},
  author={Hu, Zheyuan and Jagtap, Ameya D and Karniadakis, George Em and Kawaguchi, Kenji},
  journal={SIAM Journal on Scientific Computing},
  year={2022},
  volume={44},
  number={5},
  pages={A3158--A3182},
  doi={10.1137/21M1447039},
}

@article{jiang2026global,
  title={Global Convergence of Deep Galerkin and {PINN} Methods for Solving Partial Differential Equations},
  author={Jiang, Deqing and Sirignano, Justin and Cohen, Samuel N},
  journal={SIAM Journal on Financial Mathematics},
  volume={17},
  number={2},
  pages={620--645},
  year={2026},
  publisher={SIAM},
  doi={10.1137/24M1701502},
}

@inproceedings{eiras2023efficient,
  title={Efficient Error Certification for Physics-Informed Neural Networks},
  author={Eiras, Francisco and Bibi, Adel and Bunel, Rudy R and Dvijotham, Krishnamurthy Dj and Torr, Philip and Kumar, M. Pawan},
  year={2024},
  booktitle={Proceedings of the 41st International Conference on Machine Learning},
  pages={12318--12347},
  volume={235},
  series={Proceedings of Machine Learning Research},
  publisher={PMLR},
  url={https://proceedings.mlr.press/v235/eiras24a.html},
}

@inproceedings{rahaman2019spectral,
  title={On the spectral bias of neural networks},
  author={Rahaman, Nasim and Baratin, Aristide and Arpit, Devansh and Draxler, Felix and Lin, Min and Hamprecht, Fred and Bengio, Yoshua and Courville, Aaron},
  booktitle={Proceedings of the 36th International Conference on Machine Learning},
  pages={5301--5310},
  year={2019},
  volume={97},
  series={Proceedings of Machine Learning Research},
  publisher={PMLR},
  url={https://proceedings.mlr.press/v97/rahaman19a.html},
}

@article{wu2023comprehensive,
  title={A comprehensive study of non-adaptive and residual-based adaptive sampling for physics-informed neural networks},
  author={Wu, Chenxi and Zhu, Min and Tan, Qinyang and Kartha, Yadhu and Lu, Lu},
  journal={Computer Methods in Applied Mechanics and Engineering},
  volume={403},
  pages={115671},
  year={2023},
  publisher={Elsevier},
  doi={10.1016/j.cma.2022.115671},
}

@inproceedings{daw2022mitigating,
  title={Mitigating Propagation Failures in Physics-informed Neural Networks using Retain-Resample-Release ({R3}) Sampling},
  author={Daw, Arka and Bu, Jie and Wang, Sifan and Perdikaris, Paris and Karpatne, Anuj},
  year={2023},
  booktitle={Proceedings of the 40th International Conference on Machine Learning},
  pages={7264--7302},
  volume={202},
  series={Proceedings of Machine Learning Research},
  publisher={PMLR},
  url={https://proceedings.mlr.press/v202/daw23a.html},
}

@article{wang2024piratenets,
  title={{PirateNets}: Physics-informed deep learning with residual adaptive networks},
  author={Wang, Sifan and Li, Bowen and Chen, Yuhan and Perdikaris, Paris},
  journal={Journal of Machine Learning Research},
  volume={25},
  number={402},
  pages={1--51},
  year={2024},
  url={https://jmlr.org/papers/v25/24-0313.html},
}

@article{moseley2023finite,
  title={Finite basis physics-informed neural networks ({FBPINNs}): a scalable domain decomposition approach for solving differential equations},
  author={Moseley, Ben and Markham, Andrew and Nissen-Meyer, Tarje},
  journal={Advances in Computational Mathematics},
  volume={49},
  number={4},
  pages={62},
  year={2023},
  publisher={Springer},
  doi={10.1007/s10444-023-10065-9},
}

@article{botvinick2026ab,
  title={{AB-PINNs}: Adaptive-basis physics-informed neural networks for residual-driven domain decomposition},
  author={Botvinick-Greenhouse, Jonah and Ali, Wael H and Benosman, Mouhacine and Mowlavi, Saviz},
  journal={Machine Learning: Science and Technology},
  year={2026},
  volume={7},
  pages={045025},
  doi={10.1088/2632-2153/ae8638},
}

@article{babuvska1994p,
  title={The p and h-p versions of the finite element method, basic principles and properties},
  author={Babu{\v{s}}ka, Ivo and Suri, Manil},
  journal={SIAM Review},
  volume={36},
  number={4},
  pages={578--632},
  year={1994},
  publisher={SIAM},
  doi={10.1137/1036141},
}

@article{kharazmi2021hp,
  title={{hp-VPINNs}: Variational physics-informed neural networks with domain decomposition},
  author={Kharazmi, Ehsan and Zhang, Zhongqiang and Karniadakis, George Em},
  journal={Computer Methods in Applied Mechanics and Engineering},
  volume={374},
  pages={113547},
  year={2021},
  publisher={Elsevier},
  doi={10.1016/j.cma.2020.113547},
}
\bibliographystyle{iclr2027_conference}

\clearpage
\appendix

\addtocontents{toc}{\protect\setcounter{tocdepth}{2}}

\begin{center}
\textbf{\Large Supplementary Material}\\[0.4em]
\textbf{Physics-Informed Neural Plasticity: PDE Solvers That Reshape Themselves}
\end{center}
\label{sec:appendix}

\noindent\rule{\textwidth}{0.6pt}
\tableofcontents
\noindent\rule{\textwidth}{0.6pt}

\section{Detailed Representation and Adaptation Rules}
\label{app:construction}

The following construction expands the representation and component-update rules summarized in Section~\ref{sec:method}.

\subsection{Problem Formulation}
Let \(\Omega\subset\mathbb{R}^{d}\) denote a spatial or spatiotemporal
domain, and consider a differential equation of the form
$\mathcal{N}[u](\mathbf{z}) = 0,
    \qquad
    \mathbf{z}\in\Omega$,
subject to initial or boundary conditions
$\mathcal{B}[u](\mathbf{z}) = g(\mathbf{z}),
    \qquad
    \mathbf{z}\in\partial\Omega$.
We denote the learned solution by \(u_{\Theta}\), where
\(\Theta\) contains both the continuously optimized parameters of the
base model and the parameters of the adaptive Gaussian representation.
The corresponding physics residual is $r_{\Theta}(\mathbf{z}) =
\mathcal{N}[u_{\Theta}](\mathbf{z}).$

The adaptation mechanism is independent of the precise composition of the
base training objective. We therefore denote that objective abstractly by
\(\mathcal{L}_{\mathrm{base}}(\Theta)\). Between structural adaptation
events, all currently active continuous parameters are optimized by
gradient-based updates to \(\mathcal{L}_{\mathrm{base}}\).

For example, the FM3D benchmark uses coordinates
\(\mathbf{z}=(t,x,y)\), with linear transport residual
\begin{equation}
    r_{\Theta}(t,x,y)
    =
    \partial_t u_{\Theta}(t,x,y)
    +
    a\,\partial_x u_{\Theta}(t,x,y)
    +
    b\,\partial_y u_{\Theta}(t,x,y),
    \label{app:eq:implemented-residual}
\end{equation}
\eqref{app:eq:implemented-residual} is the
residual for FM3D. For the other benchmarks, the same adaptation rules use
the corresponding PDE residual specified in Appendix~\ref{sec:dataset}.

\subsection{Adaptive Gaussian Local Representation}
\label{app:sec:representation}

At a given optimization stage, the representation contains
\(M\) Gaussian-localized components
$\mathcal{G}
=
\left\{
G_m
\right\}_{m=1}^{M}$.
For concise notation, we write
$G_m
    =
    \left(
        \boldsymbol{\mu}_m,
        \boldsymbol{\sigma}_m,
        g_m,
        \mathbf{c}_m,
        p_m
    \right)$,
where \(\boldsymbol{\mu}_m\in\mathbb{R}^{d}\) is the component center,
\(\boldsymbol{\sigma}_m\in\mathbb{R}_{+}^{d}\) is its diagonal scale,
\(g_m\in\mathbb{R}\) is a gate logit,
\(\mathbf{c}_m\) contains local polynomial coefficients, and
\(p_m\) is the currently active polynomial order. 

After absorbing the fixed coordinate normalization used by the model into
\(\mathbf{z}\), the normalized coordinate of component \(m\) is
$\boldsymbol{\xi}_m(\mathbf{z}) =
    \left(
        \mathbf{z}-\boldsymbol{\mu}_m
    \right)
    \oslash
    \boldsymbol{\sigma}_m$,
where \(\oslash\) denotes element-wise division. The component represents
a local polynomial
\begin{equation}
    \ell_m(\mathbf{z})
    =
    \sum_{\boldsymbol{\alpha}\in\mathcal{A}_{p_m}}
    c_{m,\boldsymbol{\alpha}}\,
    \boldsymbol{\xi}_m(\mathbf{z})^{\boldsymbol{\alpha}},
    \qquad
    \mathcal{A}_{p_m}
    =
    \left\{
        \boldsymbol{\alpha}\in\mathbb{N}_{0}^{d}
        :
        \lvert\boldsymbol{\alpha}\rvert\leq p_m
    \right\}.
    \label{app:eq:local-polynomial}
\end{equation}
Here,
\(\boldsymbol{\xi}^{\boldsymbol{\alpha}}
=\prod_{q=1}^{d}\xi_q^{\alpha_q}\).
In the three-dimensional setting, the basis contains all monomials up to total degree three, including quadratic and cubic cross terms, while the local order \(p_m\) determines which basis terms are active for each component.

Let \(\pi_m(\mathbf{z})\ge0\) denote the effective responsibility of component \(m\). Any additional internal branch or channel index is averaged out when computing the adaptation statistics, including responsibilities, centers, local values, and local gradients, and is omitted from the notation for clarity. Since the adaptation procedure is agnostic to the specific responsibility and output-aggregation rules, we express the model prediction abstractly as
$u_{\Theta}(\mathbf{z})
    =
    \mathcal{A}_{\Theta}
    \left(
        \left\{
            \pi_m(\mathbf{z}),
            \ell_m(\mathbf{z}),
            g_m
        \right\}_{m=1}^{M}
    \right)$.

Our adaptation procedure requires only that the base model expose the prediction \(u_{\Theta}\), the effective responsibilities \(\{\pi_m\}\), and, when necessary, the component-wise local values \(\{\ell_m\}\).

\subsection{Physics-Guided Local Error Estimation}
\label{app:sec:indicator}

The global physics loss measures aggregate equation and condition mismatch, but does not identify which components require additional capacity. We localize these quantities using the learned component responsibilities; a solution-error interpretation additionally requires the stability assumptions in Appendix~\ref{sec:further_theory}.

Let
\(\mathcal{C}=\{\mathbf{z}_i\}_{i=1}^{N_c}\)
denote the collocation set. The component-wise residual indicator is
\begin{equation}
    \eta_m^{\mathrm{res}}
    =
    \frac{
        \sum_{i=1}^{N_c}
        \pi_m(\mathbf{z}_i)
        \left|
            r_{\Theta}(\mathbf{z}_i)
        \right|^{2}
    }{
        \sum_{i=1}^{N_c}
        \pi_m(\mathbf{z}_i)
        +
        \varepsilon
    }.
    \label{app:eq:eta-res}
\end{equation}
Thus, a large value of
\(\eta_m^{\mathrm{res}}\)
indicates a large weighted PDE residual in the region for which
component \(m\) is responsible.

When value-based initial or boundary observations are available, we
add a condition-error indicator. Let
\(\mathcal D_\Gamma=\{(\mathbf{z}_i^\Gamma,u_i^\Gamma)\}_{i=1}^{N_\Gamma}\)
denote the condition points. Then
\begin{equation}
    \eta_m^{\mathrm{bdry}}
    =
    \frac{
        \sum_{i=1}^{N_\Gamma}
        \pi_m(\mathbf{z}_i^\Gamma)
        \left|
            u_{\Theta}(\mathbf{z}_i^\Gamma)
            -
            u_i^\Gamma
        \right|^{2}
    }{
        \sum_{i=1}^{N_\Gamma}
        \pi_m(\mathbf{z}_i^\Gamma)
        +
        \varepsilon
    }.
    \label{app:eq:eta-boundary}
\end{equation}
In the implementation, initial and boundary points are combined when the
initial condition is imposed softly. If the initial condition is already
hard-constrained by the model, only the remaining boundary points are used
in ~\eqref{app:eq:eta-boundary}.

Residual error alone may fail to identify neighboring components whose
local approximations are mutually inconsistent. We therefore construct a
\(k\)-nearest-neighbor graph between Gaussian centers. Let
\(\mathcal{N}_k(m)\) denote the neighbors of component \(m\). For a pair
\((m,n)\), define their overlap-weighted gradient discrepancy as
\begin{equation}
    J_{mn}
    =
    \frac{
        \sum_{i=1}^{N_j}
        \pi_m(\mathbf{z}_i)
        \pi_n(\mathbf{z}_i)
        \left\|
            \nabla \ell_m(\mathbf{z}_i)
            -
            \nabla \ell_n(\mathbf{z}_i)
        \right\|_2^2
    }{
        \sum_{i=1}^{N_j}
        \pi_m(\mathbf{z}_i)
        \pi_n(\mathbf{z}_i)
        +
        \varepsilon
    }.
    \label{app:eq:pair-jump}
\end{equation}
The component-wise inconsistency indicator is then
\begin{equation}
    \eta_m^{\mathrm{jump}}
    =
    \frac{1}{\lvert\mathcal{N}_k(m)\rvert}
    \sum_{n\in\mathcal{N}_k(m)}
    J_{mn}.
    \label{app:eq:eta-jump}
\end{equation}

The product
\(\pi_m\pi_n\)
weights the comparison toward points at which both components are relevant.
The denominator regularizer attenuates pairs with negligible overlap mass.
Gradients are taken in a common global coordinate system, including the
chain-rule factors from each component's local coordinates. This is a local
disagreement score, not a literal interface jump of the global predictor.

The complete adaptation score is
\begin{equation}
    \eta_m
    =
    \eta_m^{\mathrm{res}}
    +
    \lambda_{\mathrm{bdry}}
    \eta_m^{\mathrm{bdry}}
    +
    \lambda_{\mathrm{jump}}
    \eta_m^{\mathrm{jump}},
    \label{app:eq:eta-total}
\end{equation}
where $\lambda_{\mathrm{bdry}},\lambda_{\mathrm{jump}}\ge0$ and either optional term can be disabled. Unlike a global loss,
~\eqref{app:eq:eta-total} assigns the current physics error to
individual learned components and therefore provides a criterion for
targeted changes in representation complexity.

\subsection{Adaptive Representation Refinement and Coarsening}
\label{app:sec:adaptation}

The localized error indicators in Section~\ref{app:sec:indicator} determine how
representation capacity is redistributed across the domain. At each adaptation
event, we first compute the robust reference value
\[
    \widetilde{\eta}
    =
    \operatorname{median}
    \left\{
        \eta_m
    \right\}_{m=1}^{M}.
\]
The median limits the influence of a small number of extreme residuals and
makes the adaptation criteria relative to the current distribution of
component-wise errors. Based on these statistics, the representation is
modified through polynomial enrichment, pruning, merging, and Gaussian
splitting.

\paragraph{Polynomial-order refinement.}
Before introducing additional Gaussian components, we first attempt to increase
the approximation order of high-error components. A component satisfying
\begin{equation}
    \eta_m
    >
    \tau_{\mathrm{split}}
    \widetilde{\eta}
    \label{app:eq:p-candidate}
\end{equation}
is eligible for local polynomial enrichment, with
\begin{equation}
    p_m
    \leftarrow
    \min
    \left\{
        p_m+1,
        p_{\max}
    \right\}.
    \label{app:eq:p-refine}
\end{equation}

When the pre-update degree satisfies $p_m<p_{\max}$, this operation activates the next group of preallocated
polynomial terms without increasing the number of Gaussian components. It
therefore provides local $p$-adaptation: the active polynomial basis expands
while the Gaussian geometry and component count remain unchanged. Setting
the newly activated coefficients to zero preserves the current predictor.

\paragraph{Gate-based pruning.}

Pruning targets components with small gate probabilities. Let
\[
    \gamma_m
    =
    \operatorname{sigmoid}(g_m)
\]
denote the effective gate probability, averaged over internal branches when
applicable. Component $m$ becomes a pruning candidate when
\begin{equation}
    \gamma_m
    <
    \tau_{\mathrm{gate}}.
    \label{app:eq:prune-gate}
\end{equation}
An optional conservative criterion additionally requires the component to have
low localized error,
\[
    \eta_m
    <
    \tau_{\mathrm{merge}}
    \widetilde{\eta}.
\]

When the number of candidates exceeds the removal budget,
components with the smallest gates are prioritized, with localized error used
as an additional ordering criterion. Pruning is constrained so that at least
$M_{\min}$ components remain, preventing excessive collapse of the
representation.

\paragraph{Functional merging.}
Pruning removes weakly active components individually, whereas merging
coarsens redundant components that represent similar local functions.
Low-error components satisfying
\[
    \eta_m
    <
    \tau_{\mathrm{merge}}
    \widetilde{\eta}
\]
are considered for merging. Importantly, geometric proximity alone is not
sufficient: two components must exhibit both non-negligible effective overlap
and similar local predictions.

Using a set of evaluation points
$\{\mathbf{z}_q\}_{q=1}^{N_e}$, we define the effective overlap between
components $i$ and $j$ as
\begin{equation}
    O_{ij}
    =
    \frac{1}{N_e}
    \sum_{q=1}^{N_e}
    \pi_i(\mathbf{z}_q)
    \pi_j(\mathbf{z}_q),
    \label{app:eq:merge-overlap}
\end{equation}
and their overlap-weighted functional discrepancy as
\begin{equation}
    D_{ij}
    =
    \frac{
        \sum_{q=1}^{N_e}
        \pi_i(\mathbf{z}_q)
        \pi_j(\mathbf{z}_q)
        \left|
            \ell_i(\mathbf{z}_q)
            -
            \ell_j(\mathbf{z}_q)
        \right|^2
    }{
        \sum_{q=1}^{N_e}
        \pi_i(\mathbf{z}_q)
        \pi_j(\mathbf{z}_q)
        +
        \varepsilon
    }.
    \label{app:eq:merge-similarity}
\end{equation}
A pair is accepted only if
\begin{equation}
    O_{ij}
    >
    \tau_{\mathrm{overlap}},
    \qquad
    D_{ij}
    <
    \tau_{\mathrm{sim}}.
    \label{app:eq:merge-tests}
\end{equation}
Thus, merging is restricted to components that are simultaneously relevant in
the same region and already describe similar local functions. Candidate pairs
are ranked according to large overlap, small functional discrepancy, and small
combined error. Selected pairs are disjoint, so each component participates in
at most one merge during an adaptation event.

To construct the merged component, define the average effective responsibilities
and corresponding regularized weights as
\begin{equation}
    q_m
    =
    \frac{1}{N_e}
    \sum_{q=1}^{N_e}
    \pi_m(\mathbf{z}_q),
    \qquad
    w_i
    =
    \frac{q_i}{q_i+q_j+\varepsilon},
    \qquad
    w_j
    =
    \frac{q_j}{q_i+q_j+\varepsilon}.
    \label{app:eq:merge-weights}
\end{equation}

The merged center and diagonal variance use the following moment-based update.
It gives exact weighted moment matching when $\varepsilon=0$ and
$q_i+q_j>0$; positive $\varepsilon$ makes the weights sum to less than one:
\begin{align}
    \boldsymbol{\mu}_{ij}
    &=
    w_i\boldsymbol{\mu}_i
    +
    w_j\boldsymbol{\mu}_j,
    \label{app:eq:merge-mean}
    \\
    \boldsymbol{\sigma}_{ij}^{\odot 2}
    &=
    w_i
    \left[
        \boldsymbol{\sigma}_i^{\odot 2}
        +
        \left(
            \boldsymbol{\mu}_i-\boldsymbol{\mu}_{ij}
        \right)^{\odot 2}
    \right]
    \nonumber\\
    &\quad+
    w_j
    \left[
        \boldsymbol{\sigma}_j^{\odot 2}
        +
        \left(
            \boldsymbol{\mu}_j-\boldsymbol{\mu}_{ij}
        \right)^{\odot 2}
    \right].
    \label{app:eq:merge-variance}
\end{align}
The polynomial coefficients and gate probabilities are combined using the same
weights, while the larger polynomial order is retained:
\begin{align}
    \mathbf{c}_{ij}
    &=
    w_i\mathbf{c}_i+w_j\mathbf{c}_j,
    \\
    \gamma_{ij}
    &=
    w_i\gamma_i+w_j\gamma_j,
    \qquad
    g_{ij}
    =
    \operatorname{logit}(\gamma_{ij}),
    \\
    p_{ij}
    &=
    \max\{p_i,p_j\}.
\end{align}

Any associated latent code is likewise averaged. The functional-similarity test
in Equation~\eqref{app:eq:merge-tests} screens candidates before this heuristic
initialization. It does not imply function preservation: coefficient vectors
refer to different local coordinates, and the merged geometry and responsibilities
also differ. The quiet-child perturbation theorem does not cover this merge rule.

\paragraph{Residual-guided Gaussian splitting.}

Complementing polynomial enrichment, splitting introduces additional localized
representation capacity in unresolved regions. The number
of newly created components is limited by the remaining capacity
$M_{\max}-M$, a fractional split budget, and a hard per-adaptation cap. These
constraints prevent a single adaptation event from producing a large
architectural change.

\emph{Largest-scale direction.}
A simple geometric strategy selects a high-error parent and splits it along its
largest spatial or spatiotemporal scale. Specifically,
\begin{equation}
    k_m^\star
    =
    \arg\max_{q\in\{1,\ldots,d\}}
    \sigma_{m,q},
    \label{app:eq:largest-axis}
\end{equation}
and the corresponding split direction is
$\mathbf{v}_m=\mathbf{e}_{k_m^\star}$. This produces anisotropic refinement
along the direction in which the parent Gaussian currently has the broadest
support.

\emph{Residual-covariance direction.}
A more informative alternative uses the distribution of unresolved physics
error to determine the split direction. In this mode, split candidates are
selected using the residual-only indicator $\eta_m^{\mathrm{res}}$. For each
candidate component $m$, define its residual-weighted responsibility as
\begin{equation}
    \omega_{im}
    =
    \pi_m(\mathbf{z}_i)
    \left|
        r_{\Theta}(\mathbf{z}_i)
    \right|^2.
    \label{app:eq:residual-weight}
\end{equation}
The corresponding residual-weighted center and covariance are
\begin{align}
    \overline{\mathbf{z}}_m
    &=
    \frac{
        \sum_i
        \omega_{im}\mathbf{z}_i
    }{
        \sum_i\omega_{im}
        +
        \varepsilon
    },
    \label{app:eq:residual-center}
    \\
    \mathbf{C}_m
    &=
    \frac{
        \sum_i
        \omega_{im}
        \left(
            \mathbf{z}_i-\overline{\mathbf{z}}_m
        \right)
        \left(
            \mathbf{z}_i-\overline{\mathbf{z}}_m
        \right)^\top
    }{
        \sum_i\omega_{im}
        +
        \varepsilon
    }.
    \label{app:eq:residual-covariance}
\end{align}
The split direction is chosen as the principal eigenvector
\begin{equation}
    \mathbf{v}_m
    \in
    \arg\max_{\|\mathbf{v}\|_2=1}
    \mathbf{v}^{\top}
    \mathbf{C}_m
    \mathbf{v}.
    \label{app:eq:principal-direction}
\end{equation}
Hence, refinement is oriented along the direction in which the
responsibility-weighted residual exhibits the largest spread, rather than
along a random or purely geometric direction. If the residual mass is
negligible or the covariance is numerically invalid, the implementation falls
back to a geometric direction, typically the largest-scale axis. When only a
subset of residual points is used, these points can additionally be selected
deterministically according to their residual energy and responsibility, avoiding additional randomness from point subsampling.

For either splitting strategy, we define the directional scale of the parent
by
\begin{equation}
    s_m
    =
    \sqrt{
        \mathbf{v}_m^\top
        \operatorname{diag}
        \left(
            \boldsymbol{\sigma}_m^{\odot 2}
        \right)
        \mathbf{v}_m
    }.
    \label{app:eq:directional-scale}
\end{equation}
A tentative displacement is
$\alpha_{\mathrm{split}}s_m\mathbf{v}_m$. To prevent the newborn component
from being placed far outside the effective support of its parent, we apply the
scale-normalized clamp
\[
    \operatorname{Clamp}_{\boldsymbol{\sigma}}
    (\boldsymbol{\delta})
    =
    \boldsymbol{\delta}
    \min
    \left\{
        1,\,
        \frac{
            \delta_{\max}
        }{
            \|
                \boldsymbol{\delta}
                \oslash
                \boldsymbol{\sigma}
            \|_2
            +
            \varepsilon
        }
    \right\}.
\]
The final displacement is therefore
\begin{equation}
    \Delta\boldsymbol{\mu}_m
    =
    \operatorname{Clamp}_{\boldsymbol{\sigma}_m}
    \left(
        \alpha_{\mathrm{split}}
        s_m\mathbf{v}_m
    \right).
    \label{app:eq:split-displacement}
\end{equation}

By default, the trained parent center is used as the split anchor,
$\mathbf{a}_m=\boldsymbol{\mu}_m$. The residual-weighted center
$\overline{\mathbf{z}}_m$ may alternatively be used, while constraining its
displacement from the parent in normalized parent-scale units. The child
geometry is then initialized as
\begin{equation}
    \boldsymbol{\mu}_{m'}
    =
    \mathbf{a}_m
    +
    \Delta\boldsymbol{\mu}_m,
    \qquad
    \boldsymbol{\sigma}_{m'}
    =
    \operatorname{clip}
    \left(
        \beta_{\sigma}
        \boldsymbol{\sigma}_m,\,
        \sigma_{\min},\,
        \sigma_{\max}
    \right),
    \label{app:eq:child-geometry}
\end{equation}
where $0<\beta_\sigma<1$ and $0<\sigma_{\min}\le\sigma_{\max}$. Reducing the scale of the child enables it to
specialize to a smaller unresolved region.

\paragraph{Near-function-preserving split initialization.}
A structural refinement should ideally increase future representation capacity
without causing a large instantaneous perturbation to the learned solution.
Simply copying the parent polynomial coefficients to a displaced and rescaled
child does not achieve this, because the coefficients are defined relative to
the component's normalized local coordinates.

Let the parent and child coordinates be
\[
    \boldsymbol{\xi}_p
    =
    \left(
        \mathbf{z}-\boldsymbol{\mu}_p
    \right)
    \oslash
    \boldsymbol{\sigma}_p,
    \qquad
    \boldsymbol{\xi}_c
    =
    \left(
        \mathbf{z}-\boldsymbol{\mu}_c
    \right)
    \oslash
    \boldsymbol{\sigma}_c.
\]
These coordinates satisfy the affine relation
\begin{equation}
    \boldsymbol{\xi}_p
    =
    \mathbf{d}
    +
    \mathbf{b}
    \odot
    \boldsymbol{\xi}_c,
    \qquad
    \mathbf{d}
    =
    \left(
        \boldsymbol{\mu}_c-\boldsymbol{\mu}_p
    \right)
    \oslash
    \boldsymbol{\sigma}_p,
    \qquad
    \mathbf{b}
    =
    \boldsymbol{\sigma}_c
    \oslash
    \boldsymbol{\sigma}_p.
    \label{app:eq:coordinate-transform}
\end{equation}

Suppose the active parent polynomial is
\[
    \ell_p(\boldsymbol{\xi}_p)
    =
    \sum_{\boldsymbol{\alpha}\in\mathcal{A}_{p}}
    c^{p}_{\boldsymbol{\alpha}}
    \boldsymbol{\xi}_p^{\boldsymbol{\alpha}}.
\]
Substituting Equation~\eqref{app:eq:coordinate-transform} and applying the
multi-index binomial expansion yields the child coefficients
\begin{equation}
    c^{c}_{\boldsymbol{\beta}}
    =
    \sum_{\substack{
        \boldsymbol{\alpha}\in\mathcal{A}_{p}\\
        \boldsymbol{\alpha}\succeq\boldsymbol{\beta}
    }}
    c^{p}_{\boldsymbol{\alpha}}
    \binom{\boldsymbol{\alpha}}{\boldsymbol{\beta}}
    \mathbf{d}^{\boldsymbol{\alpha}-\boldsymbol{\beta}}
    \mathbf{b}^{\boldsymbol{\beta}},
    \label{app:eq:coefficient-translation}
\end{equation}
where
\[
    \binom{\boldsymbol{\alpha}}{\boldsymbol{\beta}}
    =
    \prod_{q=1}^{d}
    \binom{\alpha_q}{\beta_q},
\]
and $\boldsymbol{\alpha}\succeq\boldsymbol{\beta}$ denotes component-wise
inequality. For a complete monomial basis through the active degree, this
translation makes the child polynomial represent the same function of the
global coordinate as the parent polynomial at initialization.

Preserving the local polynomial does not, however, guarantee that the complete
model prediction is exactly unchanged because the output aggregation also
depends on the Gaussian responsibilities and gates. To further suppress the
perturbation caused by inserting the child, the default refinement leaves the
trained parent unchanged and initializes the child with a small gate
probability. Let $\gamma_p=\operatorname{sigmoid}(g_p)$. The child probability
is initialized schematically as
\begin{equation}
    \gamma_c
    =
    \operatorname{clip}
    \left(
        \min
        \left\{
            \rho_g\gamma_p,\,
            \gamma_0,\,
            \gamma_{\max}
        \right\},
        \gamma_{\min},\,
        \gamma_{\max}
    \right),
    \qquad
    g_c
    =
    \operatorname{logit}(\gamma_c),
    \label{app:eq:quiet-child}
\end{equation}
where $\rho_g\ll1$ and $0<\gamma_{\min}\le\gamma_{\max}<1$. Under the normalized-mixture assumptions of
Theorem~\ref{thm:quiet_child}, the perturbation is bounded in terms of the
child gate, mixture coverage, and component derivatives. The update aims for
\begin{equation}
    u_{\Theta^{+}}(\mathbf{z})
    \approx
    u_{\Theta^{-}}(\mathbf{z})
    \label{app:eq:near-preserving}
\end{equation}
at the instant of insertion. Subsequent gradient updates can then gradually
activate and specialize the newborn component in the unresolved region.

The implementation additionally supports a symmetric replacement strategy in
which the parent and child are placed on opposite sides of the split anchor,
both scales are reduced, both local polynomials are translated into their new
coordinate systems, and the original gate probability is divided between
them.
The additive quiet-child strategy is the stated default: it preserves the
trained parent parameters, and its instantaneous perturbation can be bounded
under the assumptions in Theorem~\ref{thm:quiet_child}.

\section{Further Theoretical Results}\label{sec:further_theory}
\label{app:sec:recap-theory}

We provide an a posteriori interpretation of the localized ReCAP indicators and show that the
default additive Gaussian split produces a controlled instantaneous perturbation. The analysis separates
three questions: whether the local scores control solution error, whether the residual-covariance direction
maximizes residual-weighted projected coordinate variance, and how inserting a new component perturbs the current
physics objective. The stability premise below follows the standard principle that residual accuracy implies
solution accuracy only when the underlying PDE is stable in the chosen residual norm
\citep{hu2021extended,jiang2026global}. This premise is not automatic for arbitrary PDEs or losses
\citep{wang2022l2}. The full statements and proofs are given below.

Let $u$ solve $\mathcal Lu=f$ in $Q$ and $\mathcal Bu=g$ on $\Gamma$. Let $\mu_Q,\mu_\Gamma$ be probability measures, $\kappa_\Gamma\ge0$, and $\varepsilon>0$. For a candidate $v$, assume
\begin{equation}
\|v-u\|_X^2\le C_{\rm stab}^2\left(
\|\mathcal Lv-f\|_{L^2(\mu_Q)}^2+
\kappa_\Gamma\|\mathcal Bv-g\|_{L^2(\mu_\Gamma)}^2
\right).
\label{app:eq:stability_assumption}
\end{equation}
Let the effective Gaussian responsibilities satisfy
$0\le\pi_m\le1$ and $\sum_{m=1}^M\pi_m=1$. Define
\begin{align}
s_m&=\int_Q\pi_m\,d\mu_Q,&
\eta_m^{\rm res}&=
\frac{\int_Q\pi_m|r_\Theta|^2\,d\mu_Q}{s_m+\varepsilon},\\
s_m^\Gamma&=\int_\Gamma\pi_m\,d\mu_\Gamma,&
\eta_m^{\rm bdry}&=
\frac{\int_\Gamma\pi_m|b_\Theta|^2\,d\mu_\Gamma}{s_m^\Gamma+\varepsilon},
\end{align}
where $r_\Theta=\mathcal Lu_\Theta-f$ and $b_\Theta=\mathcal Bu_\Theta-g$.

\begin{theorem}[Mass-weighted ReCAP reliability certificate]
\label{thm:recap-reliability}
Under \eqref{app:eq:stability_assumption} and the partition-of-unity condition,
\begin{equation}
\|u_\Theta-u\|_X^2
\le C_{\rm stab}^2\left[
\sum_{m=1}^M(s_m+\varepsilon)\eta_m^{\rm res}
+\kappa_\Gamma\sum_{m=1}^M(s_m^\Gamma+\varepsilon)\eta_m^{\rm bdry}
\right].
\label{app:eq:recap-reliability}
\end{equation}

Moreover, freeze $u_\Theta$ and draw independent certification samples
$Z_i\sim\mu_Q$, $i=1,\ldots,N_c$, and
$Y_j\sim\mu_\Gamma$, $j=1,\ldots,N_\Gamma$. For $N_c,N_\Gamma\ge1$, let the empirical masses and indicators be defined by replacing
the integrals above with sample sums, and set
\begin{equation}
\widehat{\mathcal E}_{\mathrm{ReCAP}}^2
=\frac1{N_c}\sum_m(\widehat s_m+\varepsilon)\widehat\eta_m^{\rm res}
+\frac{\kappa_\Gamma}{N_\Gamma}
\sum_m(\widehat s_m^\Gamma+\varepsilon)\widehat\eta_m^{\rm bdry}.
\end{equation}
If $|r_\Theta|\le R_\infty$ on $Q$ and $|b_\Theta|\le G_\infty$ on $\Gamma$, then, for $0<\delta<1$, with probability at least
$1-\delta$,
\begin{equation}
\|u_\Theta-u\|_X
\le C_{\rm stab}\left[
\widehat{\mathcal E}_{\mathrm{ReCAP}}^2
+R_\infty^2\sqrt{\frac{\log(4/\delta)}{2N_c}}
+\kappa_\Gamma G_\infty^2\sqrt{\frac{\log(4/\delta)}{2N_\Gamma}}
\right]^{1/2}.
\label{app:eq:recap-sample-certificate}
\end{equation}
\end{theorem}

\begin{proof}
By definition and the partition-of-unity identity,
\begin{align}
\sum_m(s_m+\varepsilon)\eta_m^{\rm res}
&=\sum_m\int_Q\pi_m|r_\Theta|^2\,d\mu_Q
=\int_Q\left(\sum_m\pi_m\right)|r_\Theta|^2\,d\mu_Q
=\|r_\Theta\|_{L^2(\mu_Q)}^2.
\end{align}
The same calculation gives
$\sum_m(s_m^\Gamma+\varepsilon)\eta_m^{\rm bdry}
=\|b_\Theta\|_{L^2(\mu_\Gamma)}^2$.
Substitution into \eqref{app:eq:stability_assumption} proves \eqref{app:eq:recap-reliability}.

For every realization of the certification samples, the same algebra gives
\begin{equation}
\frac1{N_c}\sum_m(\widehat s_m+\varepsilon)\widehat\eta_m^{\rm res}
=\frac1{N_c}\sum_{i=1}^{N_c}|r_\Theta(Z_i)|^2,
\end{equation}
and likewise on $\Gamma$. Since the predictor is frozen before the certification samples are drawn,
Hoeffding's inequality yields, with probability at least $1-\delta/2$,
\begin{equation}
\|r_\Theta\|_{L^2(\mu_Q)}^2
\le \frac1{N_c}\sum_i|r_\Theta(Z_i)|^2
+R_\infty^2\sqrt{\frac{\log(4/\delta)}{2N_c}},
\end{equation}
and, with probability at least $1-\delta/2$,
\begin{equation}
\|b_\Theta\|_{L^2(\mu_\Gamma)}^2
\le \frac1{N_\Gamma}\sum_j|b_\Theta(Y_j)|^2
+G_\infty^2\sqrt{\frac{\log(4/\delta)}{2N_\Gamma}}.
\end{equation}
A union bound and \eqref{app:eq:stability_assumption} prove \eqref{app:eq:recap-sample-certificate}.
\end{proof}

If $\kappa_\Gamma=0$, the condition terms and their sample-count requirements may be omitted.

Theorem~\ref{thm:recap-reliability} shows that mass-weighted local indicators reconstruct the global
physics loss. The unweighted component score remains useful for ranking refinements, but it is not itself a
global certificate. The independent-sample requirement is essential: reusing training points requires a uniform
generalization argument such as those developed for PINNs and XPINNs \citep{hu2021extended}, or a continuous-domain
verification method \citep{eiras2023efficient}.

\paragraph{Explicit transport stability.}
For the linear transport operator $\mathcal L=\partial_t+\beta\cdot\nabla$ on a space--time cylinder $Q=(0,T)\times\Omega_x$, let
$\Gamma_-=\{x\in\partial\Omega_x:\beta(x)\cdot n(x)<0\}$,
$\Sigma_-=(0,T)\times\Gamma_-$, where $n$ is the outward unit normal, and
$c_\beta=1+\|\nabla\cdot\beta\|_{L^\infty(\Omega_x)}$. If
$e=u_\Theta-u$, $r=\mathcal Le$, $e_0=e(0,\cdot)$, and $e_-=e|_{\Sigma_-}$, then
\begin{equation}
\|e\|_{L^2(Q)}^2
\le Te^{c_\beta T}\left(
\|e_0\|_{L^2(\Omega_x)}^2+
\|r\|_{L^2(Q)}^2+
\|e_-\|_{L^2(\Sigma_-;|\beta\cdot n|)}^2
\right).
\label{app:eq:transport_stability}
\end{equation}
Indeed, multiplying $\partial_te+\beta\cdot\nabla e=r$ by $2e$, integrating by parts, discarding the
nonnegative outflow term, and applying Young's inequality gives
\begin{equation}
\frac{d}{dt}\|e(t)\|_{L^2(\Omega_x)}^2
\le c_\beta\|e(t)\|_{L^2(\Omega_x)}^2
+\|r(t)\|_{L^2(\Omega_x)}^2
+\int_{\Gamma_-}|\beta\cdot n|e_-^2\,dS.
\end{equation}
Gronwall's inequality followed by integration over $t\in(0,T)$ proves \eqref{app:eq:transport_stability}.
The calculation assumes sufficient regularity for these derivatives and traces.
The displayed estimate uses volume and flux-weighted boundary measures;
normalizing these to probability measures rescales the stability constants.

\begin{proposition}[Optimal residual-covariance direction]
\label{prop:cov_direction}
For a selected component, define
$\omega_i=\pi_m(z_i)|r_\Theta(z_i)|^2$, $W=\sum_i\omega_i>0$,
$\bar z=W^{-1}\sum_i\omega_i z_i$, and
$C=W^{-1}\sum_i\omega_i(z_i-\bar z)(z_i-\bar z)^\top$.
If $v_1$ is a unit eigenvector associated with $\lambda_{\max}(C)$, then
\begin{equation}
v_1\in\operatorname*{arg\,max}_{\|v\|_2=1}
\frac1W\sum_i\omega_i\bigl(v^\top(z_i-\bar z)\bigr)^2.
\end{equation}
\end{proposition}

\begin{proof}
The objective equals $v^\top Cv$. Writing $C=V\Lambda V^\top$ and $a=V^\top v$ gives
$v^\top Cv=\sum_k\lambda_k a_k^2\le\lambda_{\max}(C)\sum_k a_k^2
=\lambda_{\max}(C)$, with equality for a leading eigenvector.
\end{proof}

We next quantify the stability of the default additive split. Assume one output channel has the normalized
gated mixture form
\begin{equation}
u^-(z)=\frac{\sum_{m=1}^Mq_m(z)\ell_m(z)}{D(z)},
\qquad q_m=\gamma_m\phi_m,
\qquad D=\sum_{m=1}^Mq_m.
\label{app:eq:mixture_assumption}
\end{equation}
This is the aggregation assumption for the perturbation result below.
A different decoder, branch combination, or hard-constraint transformation
requires bounds for the corresponding output map.

\begin{lemma}[Exact polynomial translation]
\label{lem:poly_translation}
Let $\xi_p=(z-\mu_p)\oslash\sigma_p$ and $\xi_c=(z-\mu_c)\oslash\sigma_c$.
Set $d=(\mu_c-\mu_p)\oslash\sigma_p$ and $b=\sigma_c\oslash\sigma_p$. If
$\ell_p=\sum_{|\alpha|\le p}c_\alpha^p\xi_p^\alpha$ and the complete monomial basis through degree $p$ is
used, then the child coefficients
\begin{equation}
c_\beta^c=
\sum_{\substack{|\alpha|\le p\\\alpha\succeq\beta}}
c_\alpha^p\binom{\alpha}{\beta}d^{\alpha-\beta}b^\beta
\label{app:eq:translated_coef}
\end{equation}
satisfy $\ell_c(z)=\ell_p(z)$ for every $z$.
\end{lemma}

\begin{proof}
The coordinates satisfy $\xi_p=d+b\odot\xi_c$. Substituting this identity into the parent polynomial and
applying the multi-index binomial theorem gives
\begin{equation}
\ell_p
=\sum_{|\alpha|\le p}c_\alpha^p
\sum_{\beta\preceq\alpha}\binom{\alpha}{\beta}
d^{\alpha-\beta}b^\beta\xi_c^\beta.
\end{equation}
Collecting the coefficient of each $\xi_c^\beta$ gives \eqref{app:eq:translated_coef}. Completeness of the basis
ensures that every resulting monomial is present.
\end{proof}

When component-wise local orders are used, the translation must be applied only to source coefficients whose
degree is active for the parent, or all inactive stored coefficients must be zero. Otherwise an inactive
higher-order coefficient can translate into an active lower-order coefficient, and exact preservation need not
hold.

\begin{theorem}[Shock-controlled quiet-child refinement]
\label{thm:quiet_child}
Translate a new child from parent $p$ according to Lemma~\ref{lem:poly_translation}, leave the parent
unchanged, and let $q_c=\gamma_c\phi_c$ with $0\le\gamma_c\le\epsilon_g$ and $\gamma_c$ constant in $z$. Assume $\mathcal L=\partial_t+\beta\cdot\nabla$, $0\le\phi_c\le1$ on $Q\cup\Gamma$, and that the following lower bound on $D$ holds on $Q\cup\Gamma$:
\begin{equation}
D\ge d_0>0,
\qquad
\|\mathcal L\phi_c\|_{L^\infty(Q)}\le K_\phi,
\qquad
\|\mathcal LD\|_{L^\infty(Q)}\le K_D.
\end{equation}
Set $h=\ell_p-u^-$,
$K_\alpha=K_\phi/d_0+K_D/d_0^2$, and
\begin{equation}
C_{\rm split}=\left[
\left(K_\alpha\|h\|_{L^2(Q)}+d_0^{-1}\|\mathcal Lh\|_{L^2(Q)}\right)^2
+\kappa_\Gamma d_0^{-2}\|h\|_{L^2(\Gamma)}^2
\right]^{1/2}.
\end{equation}
For
$J(v)=\|\mathcal Lv-f\|_{L^2(Q)}^2+
\kappa_\Gamma\|v-g\|_{L^2(\Gamma)}^2$, the post-split predictor satisfies
\begin{align}
\sqrt{J(u^+)}&\le\sqrt{J(u^-)}+\epsilon_gC_{\rm split},
\label{app:eq:sqrt_loss_shock}\\
J(u^+)-J(u^-)&\le
2\epsilon_gC_{\rm split}\sqrt{J(u^-)}+
\epsilon_g^2C_{\rm split}^2.
\label{app:eq:loss_shock}
\end{align}
If the child gate can approach zero, then the closure of the refined hypothesis class contains the pre-split
class and its optimal achievable loss cannot increase.
\end{theorem}

\begin{proof}
By Lemma~\ref{lem:poly_translation}, the new child represents $\ell_p$ in global coordinates. Directly from
\eqref{app:eq:mixture_assumption},
\begin{equation}
u^+-u^-=\alpha h,
\qquad
\alpha=\frac{q_c}{D+q_c}.
\label{app:eq:split_identity}
\end{equation}
Since $q_c\le\epsilon_g$ and $D\ge d_0$, $0\le\alpha\le\epsilon_g/d_0$. Moreover,
\begin{equation}
\mathcal L\alpha=
\frac{D\,\mathcal Lq_c-q_c\,\mathcal LD}{(D+q_c)^2}.
\end{equation}
The gate is constant with respect to $z$ at insertion, so
$\|\mathcal Lq_c\|_\infty\le\epsilon_gK_\phi$. Therefore
\begin{equation}
\|\mathcal L\alpha\|_\infty
\le\epsilon_g\left(\frac{K_\phi}{d_0}+\frac{K_D}{d_0^2}\right)
=\epsilon_gK_\alpha.
\end{equation}
Applying the product rule to \eqref{app:eq:split_identity} gives
\begin{align}
\|\mathcal L(u^+-u^-)\|_{L^2(Q)}
&\le\epsilon_g\left(
K_\alpha\|h\|_{L^2(Q)}+d_0^{-1}\|\mathcal Lh\|_{L^2(Q)}
\right),\\
\|u^+-u^-\|_{L^2(\Gamma)}
&\le\epsilon_gd_0^{-1}\|h\|_{L^2(\Gamma)}.
\end{align}
Thus the product norm of the residual and boundary perturbations is at most
$\epsilon_gC_{\rm split}$. Minkowski's inequality proves \eqref{app:eq:sqrt_loss_shock}; squaring it proves
\eqref{app:eq:loss_shock}. Finally, retaining every old parameter and sending the child gate logit to $-\infty$
makes $q_c\to0$, so every pre-split predictor lies in the closure of the refined class.
\end{proof}

Theorem~\ref{thm:quiet_child} applies to the additive, parent-preserving split. It does not directly cover the
optional symmetric replacement strategy. Likewise, pruning and merging require separate perturbation guards.
For local order enrichment, setting every newly activated polynomial coefficient to zero makes the insertion
exactly function-preserving and enlarges the hypothesis class.

\section{PDE Benchmarks} \label{sec:dataset}
We specify the governing equations, domains, and reference solutions or initial data for the five benchmarks. The suffix 3D or 4D counts all input coordinates, including time where applicable.

\subsection{(2+1)D Flow Mixing}
FM3D is a linear transport benchmark for the advection and rotational mixing of a passive scalar in a two-dimensional velocity field:
\begin{equation}
\begin{aligned}
    &u_t+a(x,y)u_x+b(x,y)u_y=0,
    \quad (x,y,t)\in[-4,4]^2\times[0,4],\\
    &a(x,y)=-\frac{v_{\mathrm{tan}}(r)}{v_{\max}}\frac{y}{r},
    \qquad b(x,y)=\frac{v_{\mathrm{tan}}(r)}{v_{\max}}\frac{x}{r},\\
    &v_{\mathrm{tan}}(r)=\operatorname{sech}^2(r)\tanh(r),
    \qquad r=\sqrt{x^2+y^2},\qquad v_{\max}=0.385.
\end{aligned}
\label{flow_mixing3d}
\end{equation}
The analytic solution is
\begin{equation}
u^\star(x,y,t)=-\tanh\!\left[\frac y2\cos(\omega(r)t)-\frac x2\sin(\omega(r)t)\right],
\qquad \omega(r)=\frac{v_{\mathrm{tan}}(r)}{r v_{\max}},
\end{equation}
with initial condition $u(x,y,0)=-\tanh(y/2)$ \citep{tamamidis1993evaluation}. The apparent singularities at $r=0$ are removable: $a(0,0)=b(0,0)=0$ and $\omega(0)=1/v_{\max}$. These values follow from $v_{\mathrm{tan}}(r)/r\to1$ as $r\to0$.

\subsection{3D Helmholtz Equation}
H3D is a stationary elliptic benchmark on $\Omega_x=[-1,1]^3$:
\begin{equation}
\Delta_{\mathbf x}u+k^2u=q\quad\text{in }\Omega_x,
\qquad u=0\quad\text{on }\partial\Omega_x.
\end{equation}
The manufactured solution and source are
\begin{equation}
u^\star(\mathbf x)=\prod_{i=1}^3\sin(a_i\pi x_i),
\qquad q(\mathbf x)=\left[k^2-\pi^2\sum_{i=1}^3a_i^2\right]u^\star(\mathbf x),
\end{equation}
where $k=1$ and $(a_1,a_2,a_3)=(4,4,3)$. Thus $q=(1-41\pi^2)u^\star$, and the prescribed homogeneous boundary condition is satisfied exactly by $u^\star$.

\subsection{(2+1)D Nonlinear Diffusion}
ND3D describes solution-dependent diffusion in two spatial dimensions and time:
\begin{equation}
\begin{aligned}
    &u_t=0.05\left(\|\nabla_{\mathbf x}u\|_2^2+u\Delta_{\mathbf x}u\right),
    \quad (x,y,t)\in[-1,1]^2\times[0,1],\\
    &u(\mathbf x,0)=0.25g\!\left(\mathbf x;0.2,0.3,\frac1{\sqrt{10}}\right)
    +0.4g\!\left(\mathbf x;-0.1,-0.5,\frac1{\sqrt{15}}\right)\\
    &\hspace{4.6em}+0.3g\!\left(\mathbf x;-0.5,0,\frac1{\sqrt{20}}\right),
\end{aligned}
\end{equation}
where $\mathbf x=(x,y)$ and
$g(\mathbf x;a,b,\sigma)=\exp\!\left(-[(x-a)^2+(y-b)^2]/\sigma^2\right)$.
The gradient and Laplacian act only on the spatial coordinates. The initial condition contains three localized peaks with different amplitudes, widths, and locations.

\subsection{Klein--Gordon Equations}
For KG3D and KG4D, let $d_s\in\{2,3\}$ be the number of spatial coordinates, $\Omega_x=[-1,1]^{d_s}$, and $t\in[0,10]$. The total input dimension used in Section~\ref{sec:method} is $d=d_s+1$. Define
\begin{equation}
S(\mathbf x)=\sum_{i=1}^{d_s}x_i,\qquad
P(\mathbf x)=\prod_{i=1}^{d_s}x_i,\qquad
u^\star(\mathbf x,t)=S(\mathbf x)\cos(2t)+P(\mathbf x)\sin(2t).
\end{equation}
The corresponding manufactured problem is
\begin{equation}
\begin{aligned}
    &u_{tt}-\Delta_{\mathbf x}u+u^2=f,
      &&(\mathbf x,t)\in\Omega_x\times(0,10],\\
    &u(\mathbf x,0)=S(\mathbf x),\qquad
      u_t(\mathbf x,0)=2P(\mathbf x),
      &&\mathbf x\in\Omega_x,\\
    &u(\mathbf x,t)=u^\star(\mathbf x,t),
      &&(\mathbf x,t)\in\partial\Omega_x\times[0,10],
\end{aligned}
\end{equation}
where $f=(u^\star)^2-4u^\star$, since $u^\star_{tt}=-4u^\star$ and $\Delta_{\mathbf x}u^\star=0$. Thus $d_s=2$ gives the stated KG3D solution and $d_s=3$ gives KG4D. The initial-velocity condition follows by differentiating the manufactured solution; it is part of the complete second-order-in-time problem specification.

\section{Experimental Settings} \label{sec:experiment-settings}
In this section, we summarize the experimental configurations used for all benchmarks. The key hyperparameters, including the number of training epochs and other model-specific settings, are reported in Table~\ref{tab:key_hyperparameters_all_pdes}.
\begin{table}[htbp]
    \centering
    \caption{Key hyperparameters for the benchmarks.}
    \label{tab:key_hyperparameters_all_pdes}
    \small
    \begin{tabular}{lccccc}
        \toprule
        Hyperparameter
        & FM3D
        & H3D
        & ND3D
        & KG3D
        & KG4D \\
        \midrule
        Training epochs
        & 80,000
        & 50,000
        & 100,000
        & 50,000
        & 100,000 \\

        Learning rate
        & $5\times10^{-3}$
        & $1\times10^{-3}$
        & $2\times10^{-3}$
        & $5\times10^{-3}$
        & $1\times10^{-3}$ \\

        Batch size
        & 1,024
        & 1,024
        & 1,024
        & 1,024
        & 512 \\

        Hidden features
        & 48
        & 16
        & 32
        & 48
        & 128 \\

        Initial number of Gaussians
        & 2,000
        & 3,000
        & 600
        & 3,000
        & 2,000 \\

        Minimum/maximum Gaussians
        & 600/4,000
        & 600/5,000
        & 300/1,000
        & 600/4,000
        & 200/2,000 \\

        Polynomial degree
        & 2
        & 2
        & 2
        & 2
        & 2 \\

        Gaussian scale parameter
        & 0.15
        & 0.15
        & 0.20
        & 0.15
        & 0.20 \\

        Adaptation start
        & 3,000
        & 1,000
        & 3,000
        & 2,000
        & 1,000 \\

        Adaptation interval
        & 2,000
        & 2,000
        & 2,000
        & 3,000
        & 2,000 \\

        Split fraction
        & 0.02
        & 0.03
        & 0.03
        & 0.015
        & 0.03 \\

        Adaptive sampling interval
        & 3,000
        & 3,000
        & 3,000
        & 4,000
        & 3,000 \\

        $\tau_{\mathrm{split}}$
        & 1.3
        & 1.3
        & 1.3
        & 1.5
        & 1.3 \\

        $\tau_{\mathrm{merge}}$
        & 0.7
        & 0.5
        & 0.5
        & 0.4
        & 0.5 \\
        \bottomrule
    \end{tabular}
\end{table}

\FloatBarrier
\section{Visualization} \label{sec:vis}
In this section, we provide qualitative visualizations of the predicted solutions and corresponding absolute-error distributions on representative PDE benchmarks, with comparisons against reference solutions and selected baselines in Figures~\ref{fig:FM3d}, \ref{fig:Helmholtz3d}, and \ref{fig:KG3D}.

For FM3D, ReCAP accurately captures the evolving rotational mixing structures over time and closely follows the reference solution. The corresponding error maps show that the remaining errors are mainly localized around the increasingly complex spiral interfaces, while ReCAP maintains lower and more concentrated errors than the representative baselines.

For H3D and KG3D, ReCAP also reproduces the dominant spatial patterns across different slices and time instances. In both cases, the predicted solutions remain close to the reference fields, and the absolute-error maps indicate reduced deviations compared with SPINN and PirateNets, supporting its accuracy on the tested equations.

\begin{figure}[!htbp]
    \centering
    \includegraphics[width=1\linewidth]{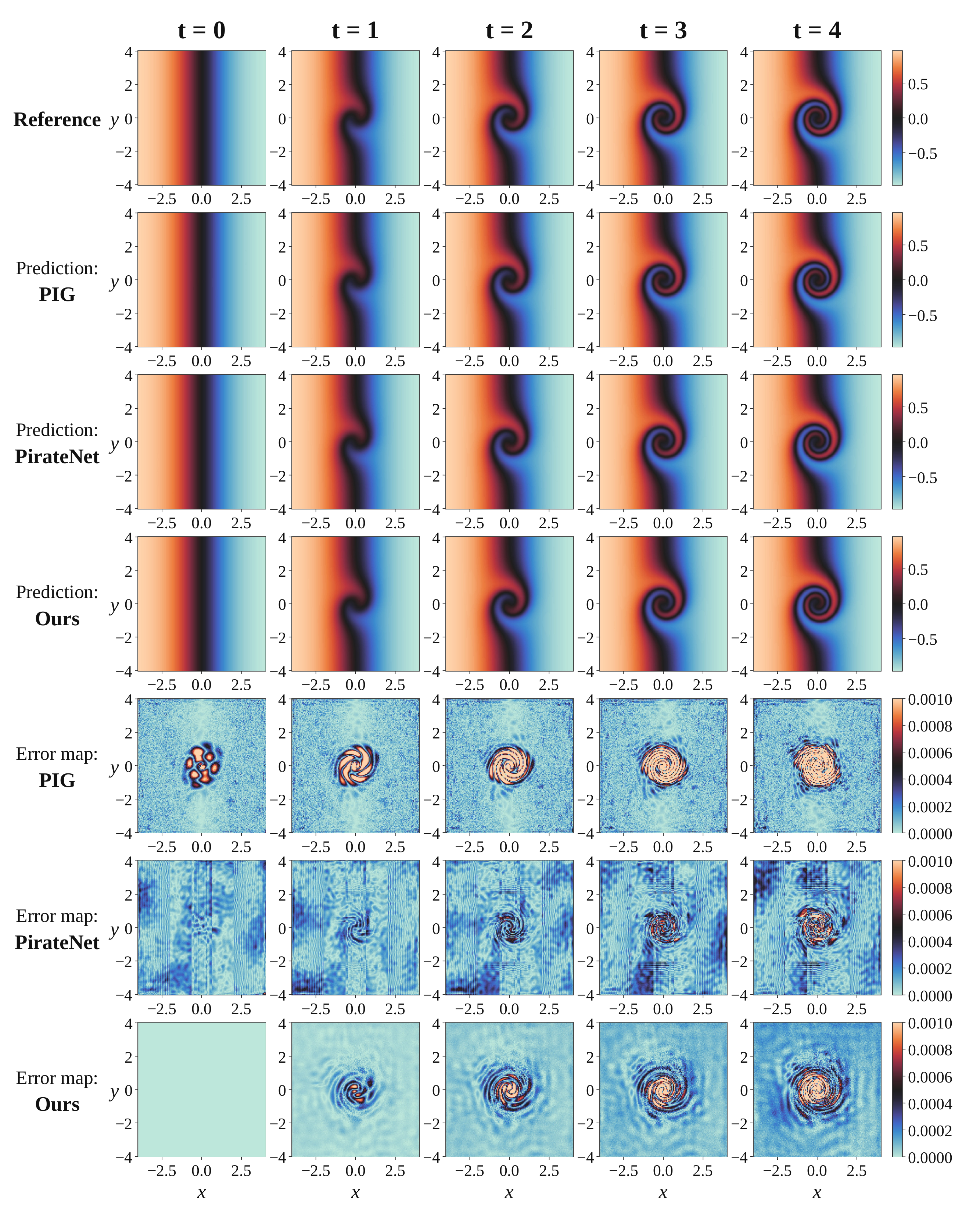}
    \caption{FM3D solution and absolute-error snapshots over the evaluated
    time points, comparing the reference field with ReCAP and representative
    baselines.}
    \label{fig:FM3d}
\end{figure}

\begin{figure}[!htbp]
    \centering
    \includegraphics[width=1\linewidth]{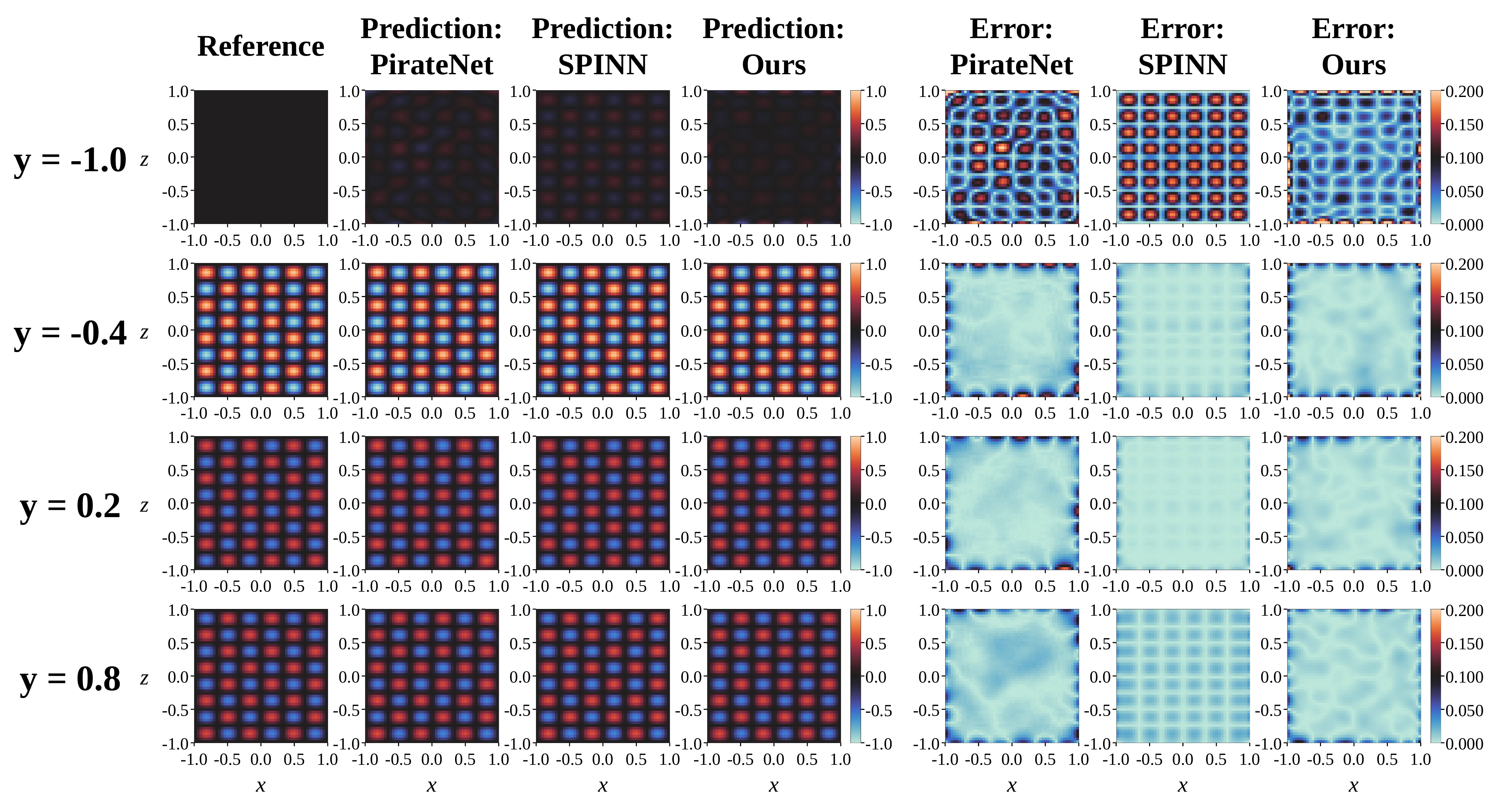}
    \caption{H3D reference solution, model predictions, and absolute-error
    maps on representative two-dimensional slices.}
    \label{fig:Helmholtz3d}
\end{figure}

\begin{figure}[!htbp]
    \centering
    \includegraphics[width=1\linewidth]{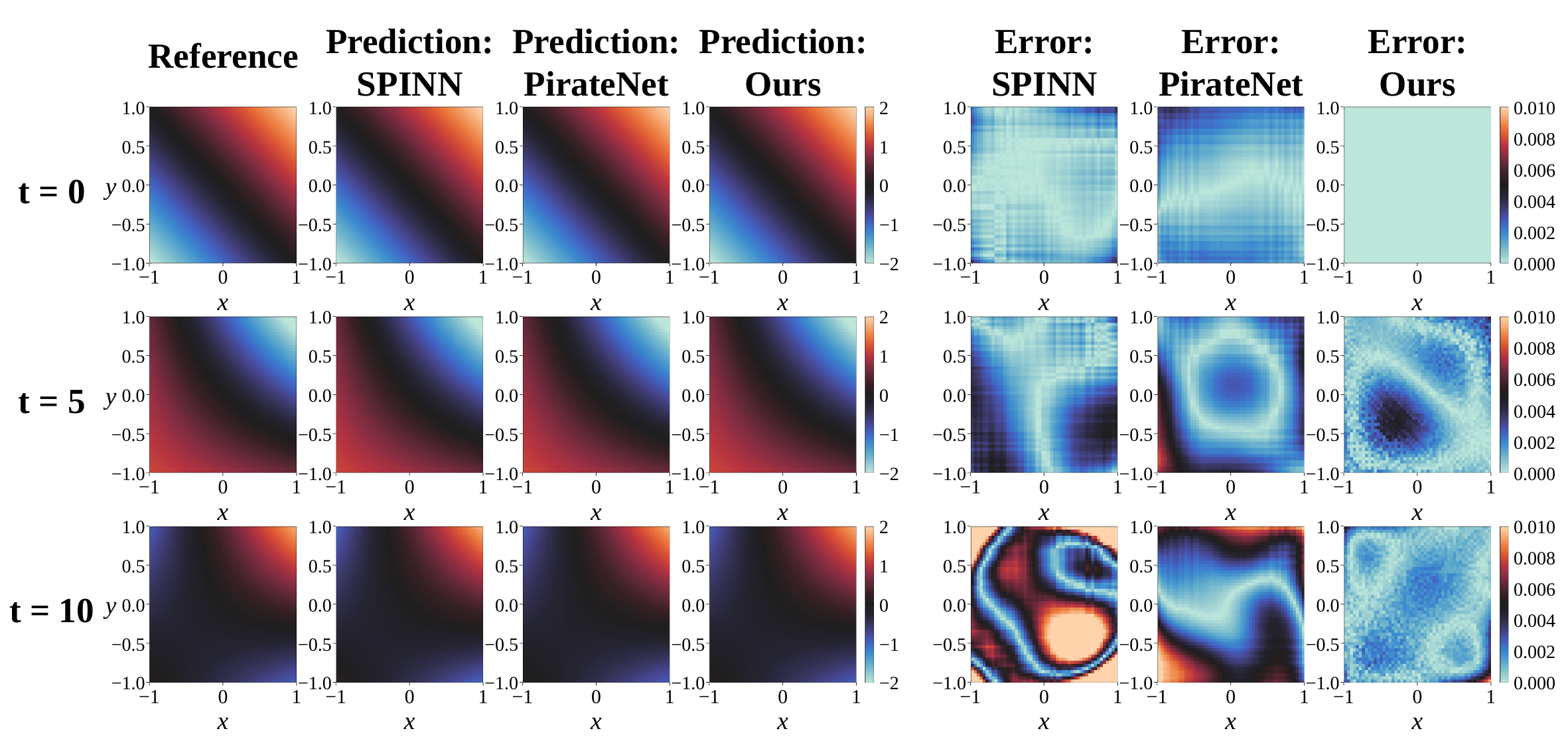}
    \caption{KG3D reference solution, model predictions, and absolute-error
    maps on representative two-dimensional slices.}
    \label{fig:KG3D}
\end{figure}

\FloatBarrier
\section{Computational Cost} \label{sec:computational-cost}
Figure~\ref{fig:computational-cost} compares reported floating-point operation counts on KG3D. ReCAP uses 0.53 GFLOPs in this comparison, matching SINN and falling below ActNet (1.64), PirateNets (2.59), and CoupledNet (4.38), while Scale-PINN uses 0.27 GFLOPs. These operation counts do not establish wall-clock training-time savings.

\begin{figure}[!htbp]
    \centering
    \includegraphics[width=0.8\linewidth]{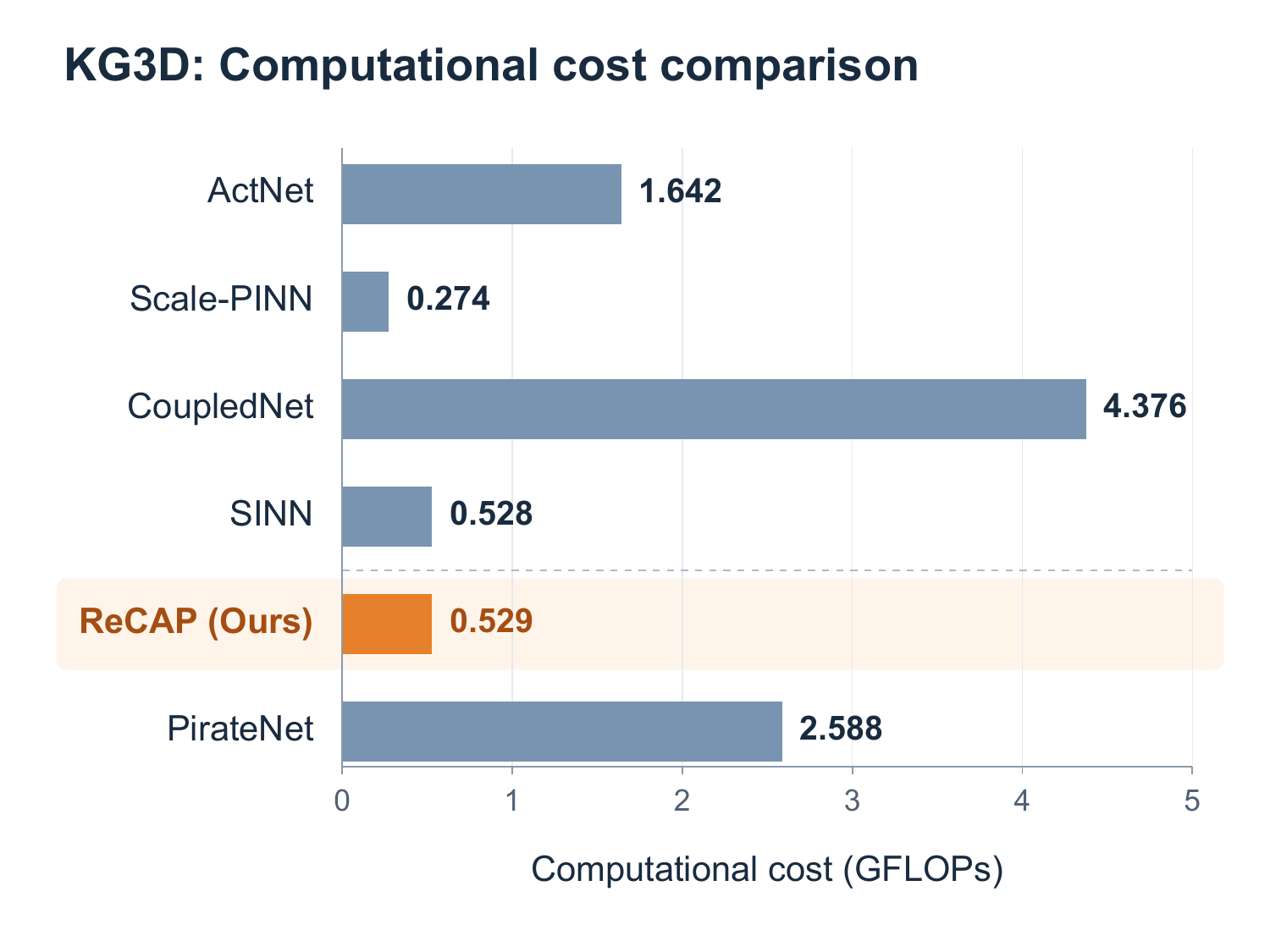}
    \caption{Reported computational cost on KG3D in GFLOPs (lower is better). ReCAP has the same reported operation count as SINN; Scale-PINN has the lowest count among the displayed methods.}
    \label{fig:computational-cost}
\end{figure}

\end{document}